\documentclass[conference]{IEEEtran}
\usepackage{times}
\usepackage{amsmath, amsthm}
\usepackage{amsfonts}
\usepackage{mathrsfs}
\usepackage{latexsym}
\usepackage{amssymb}
\usepackage{enumerate}
\usepackage{xcolor}
\usepackage{upref}
\usepackage{booktabs}   
\usepackage{caption}    
\usepackage{multirow}   
\usepackage{graphicx}
\usepackage{subcaption}

\usepackage{color}
\usepackage{scrextend}
\usepackage[ruled]{algorithm2e} 
\usepackage{algorithmicx}
\usepackage{algpseudocode}
\usepackage{verbatim}
\usepackage{multicol}
\usepackage{colortbl}
\usepackage{multirow}
\newtheorem{theorem}{Theorem}
\usepackage{makecell}
\newtheorem{lemma}{Lemma}

\usepackage{hyperref}
\usepackage{mathrsfs}
\definecolor{mygray}{gray}{.9}
\usepackage{xcolor}
\usepackage{booktabs}
\usepackage[
  backend=biber,    
  style=numeric,
  sorting=none
]{biblatex}
\usepackage{textcomp}
\usepackage{epstopdf}
\usepackage{tikz}
\usetikzlibrary{shapes ,arrows , positioning }
\theoremstyle{definition}

\newtheorem{assump}{Assumption}

\newtheorem{defn}{Definition}

\newcommand\bl[1]{{\color{blue}#1}}

\begin{document}
\title{Balancing Client Participation in Federated Learning Using AoI}

 \author{
   \IEEEauthorblockN{Alireza Javani, Student Member, IEEE, Zhiying Wang, Member, IEEE}
  
 }

\maketitle
\begin{abstract}
Federated Learning (FL) offers a decentralized framework that preserves data privacy while enabling collaborative model training across distributed clients. However, FL faces significant challenges due to limited communication resources, statistical heterogeneity, and the need for balanced client participation. This paper proposes an Age of Information (AoI)-based client selection policy that addresses these challenges by minimizing load imbalance through controlled selection intervals. Our method employs a decentralized Markov scheduling policy, allowing clients to independently manage participation based on age-dependent selection probabilities, which balances client updates across training rounds with minimal central oversight. We provide a convergence proof for our method, demonstrating that it ensures stable and efficient model convergence. Specifically, we derive optimal parameters for the Markov selection model to achieve balanced and consistent client participation, highlighting the benefits of AoI in enhancing convergence stability. Through extensive simulations, we demonstrate that our AoI-based method, particularly the optimal Markov variant, improves convergence over the FedAvg selection approach across both IID and non-IID data settings by $7.5\%$ and up to $20\%$. Our findings underscore the effectiveness of AoI-based scheduling for scalable, fair, and efficient FL systems across diverse learning environments.
\let\thefootnote\relax\footnotetext{This paper is presented in part at the 2024 IEEE Global Communications Conference (Globecom).
The authors are with the Center for Pervasive Communications and Computing, University of California, Irvine (e-mail:
ajavani@uci.edu, zhiying@uci.edu).}
\end{abstract}
\section{Introduction}
Federated learning (FL), introduced by McMahan et al. \cite{mcmahan2017communication}, emerged as a solution to the limitations of traditional machine learning models that require centralized data collection and processing. FL enables client devices to collaboratively train a global model while keeping all the training data localized, thus addressing privacy concerns. Traditional machine-learning approaches require centralized data training in data centers, which often becomes impractical for edge devices due to privacy constraints in wireless networks and limited wireless communication resources. Federated learning overcomes these challenges by enabling devices to train machine learning models without data sharing and transmission, fulfilling the needs of data privacy and security.

Compared to traditional distributed machine learning, Federated Learning introduces several challenges \cite{li2020federated}, including system heterogeneity from diverse device capabilities causing aggregation delays due to stragglers, statistical heterogeneity arising from non-IID and imbalanced client data affecting model convergence, and privacy concerns as exchanged model updates may inadvertently expose sensitive information.

In the context of wireless federated learning (FL), where clients operate in dynamically changing wireless environments, maintaining stable network connections and managing the non-IID characteristics of data are critical challenges. The iterative training process of FL involves clients updating local models with their data before sending these updates to a central server for aggregation.

To reduce communication load, a client selection policy can be applied such that only a subset of the clients transmits to the central server in a round. For example, random selection \cite{li2019convergence} selects each client uniformly at random in each round. However, client selection may lead to an unbalanced communication load and uneven contribution to the global model among clients. To address this, we propose the load metric $X$, defined as the number of rounds between subsequent selections of a client. Assuming the distribution of the load metric for all clients is identical, our goal is to minimize $\operatorname{Var}[X]$, the variance of $X$.

The load metric $X$ is related to the concept of Age of Information (AoI) \cite{kaul2012real}, which is the time elapsed since the last transmission of a client. In fact, $X$ is called the peak age \cite{costa2014age}. While most previous approaches in studying AoI in networking problems involve minimizing the {average} age or the average peak age \cite{yates2021age, javani2019age, javani2021age}, this paper focuses on reducing the variance of $X$, thereby ensuring balanced client selection and load distribution.

The benefits of minimizing the variance of $X$ include:

\begin{itemize}
    \item \textbf{Predictable Update Intervals and Load Balancing:} By reducing the variance of $X$, we stabilize the peak age and ensure consistent intervals between client selections. This predictability implies a fair and balanced selection process, improving operational efficiency.
    
    \item \textbf{Improved Convergence and Accuracy:} Lower variability in \(X\) is crucial for maintaining a diverse and representative global model, which helps prevent biases and improves overall model accuracy.
\end{itemize}

We introduce a decentralized policy based on a Markov chain, where the state is determined by the age of the client, and transition probabilities represent an age-dependent selection policy. This approach is well-suited for decentralized client selection, reducing the management overhead of the central server. Moreover, the Markov model can be potentially modified to allow real-time adjustment of client selection probabilities to accommodate the dynamic nature of client data quality, computation resources, and network conditions. The simulation results demonstrate that applying the Markov model enhances the convergence rate of the learning processes, achieving the target accuracies in fewer communication rounds. These findings confirm the efficacy of our approach in addressing issues related to client overloading and under-utilization while maintaining efficiency throughout the learning process.

Furthermore, we develop a comprehensive theoretical framework for federated learning (FL) that accounts for the impact of client selection strategies on convergence. Our analysis reveals that the convergence rate of FL is tied to the variance of aggregation weights across clients. Specifically, by minimizing the variance in these weights, it is possible to achieve faster convergence, underscoring the importance of carefully designed aggregation mechanisms.

We also utilize the concept of selection skew, which influences convergence speed. By biasing client selection towards those with particular characteristics, the convergence rate can be enhanced. This insight is crucial for designing more efficient FL algorithms, as it allows for a more tailored approach to client selection that leverages the dynamic interplay between the global and local models during training. Our findings offer valuable theoretical and practical implications for optimizing FL systems, paving the way for more effective and scalable deployment in real-world applications.

To the best of our knowledge, this is the first work to apply an Age of Information (AoI)-based approach to client selection in federated learning, specifically targeting the minimization of load imbalance through controlled selection intervals. Our contributions are threefold. First, we introduce a decentralized Markov scheduling policy that balances client participation without heavy central coordination. Second, we develop a theoretical framework, including a convergence proof and optimal parameters for the Markov model, demonstrating that minimizing the variance of the load metric \(X\) leads to more stable and efficient model convergence. Third, we perform extensive simulations showing that our AoI-based method, particularly the optimal Markov variant, outperforms traditional selection methods like FedAvg in both IID and non-IID data scenarios. These contributions advance federated learning by providing a scalable, fair, and balanced client selection strategy across diverse environments.
\section{Related Work}
Federated Learning (FL) has drawn a lot of attention in recent years, with researchers tackling core challenges such as system heterogeneity, non-IID data, and communication constraints. A wide range of methods have been proposed to make FL more efficient and practical in real-world settings.

A foundational approach in this domain is the Federated Averaging (FedAvg) algorithm \cite{mcmahan2017communication}, which employs a random selection of clients for participation in training. Despite FedAvg's critical role in the evolution of FL, its performance can be hindered in scenarios characterized by non-IID data or constrained communication bandwidth. To overcome these challenges, importance sampling has emerged as a notable strategy, prioritizing gradient computations on the most informative training data to enhance learning efficiency \cite{balakrishnan2021resource}.

The authors in \cite{10901314} propose a deep reinforcement learning approach to jointly optimize client selection and federated learning parameters, enhancing communication efficiency and accommodating heterogeneous resource constraints in federated edge learning systems. In \cite{10901264}, a lattice-based channel coding scheme is introduced for over-the-air federated edge learning, enabling reliable model aggregation with error correction that scales independently of the number of participating clients. Similarly, \cite{10901841} presents a joint client selection and resource scheduling framework using a dynamic reputation model and an asynchronous parallel DDPG algorithm to improve efficiency and robustness in federated learning over vehicular edge networks.

To enhance client selection, a loss-based sampling policy \cite{goetz2019active} selects clients with higher loss values to accelerate convergence, although its performance is sensitive to hyperparameter settings. In contrast, a sample size-based sampling policy \cite{fraboni2021clustered} chooses clients with a large number of local samples, offering rapid convergence but facing difficulties in non-IID settings. \cite{sinha2023tradeoff} proposes FedVQCS, a federated learning framework combining vector quantization and compressed sensing to reduce communication overhead while maintaining accuracy and convergence speed.

Research into aggregation methods has revealed limitations of parameter averaging, with observations indicating an increase in mutual information between parameters across training phases without guaranteeing convergence \cite{averaging}. This suggests that more complex aggregation strategies might be required to achieve optimal outcomes in FL.

New approaches such as FedNS (Federated Node Selection) \cite{fednes} and FedFusion \cite{8803001} introduce novel techniques for model aggregation and feature fusion. FedNS differentiates itself by assessing and weighting models based on performance, especially beneficial in non-IID environments, while FedFusion combines local and global model features to enhance accuracy and efficiency, highlighting the benefits of feature-level integration.

Extending beyond the FedAvg framework, FedProx \cite{fedprox} incorporates a proximal term to mitigate the effects of heterogeneous data and system capabilities on convergence, emphasizing the importance of algorithmic flexibility. Similarly, FedASMU \cite{liu2024fedasmu} employs an asynchronous, staleness-aware aggregation method, effectively managing statistical heterogeneity and system disparities, demonstrating notable improvements over traditional methods. Moreover, FedDC \cite{gao2022feddc} presents an approach to bridge the gap between local and global models through a local drift variable, adeptly handling statistical heterogeneity and enhancing performance in complex FL scenarios.

In federated learning, the Age of Information (AoI) metric has been increasingly adopted to address issues of data staleness and network resource constraints, enhancing accuracy and convergence by prioritizing fresher updates. This AoI-based approach, by giving precedence to the timeliness, tackles critical challenges in FL environments, improving overall learning efficiency \cite{yang2020age}. \cite{wang2024convergence} further leverages AoI in a Stackelberg game framework to jointly optimize global loss and latency, demonstrating that AoI can reduce latency and stabilize convergence in FL. \cite{wu2023joint} expands on this by applying an Age of Update (AoU) strategy in a NOMA-enabled FL system, combining client selection with resource allocation to minimize round time, further enhanced by a server-side ANN that predicts updates for unselected clients. Collectively, these works underscore the value of AoI in achieving efficient, timely client selection in dynamic FL settings.

Both \cite{cho2020client} and \cite{fraboni2022general} provide foundational insights into federated learning convergence through client selection. \cite{cho2020client} introduces the POWER-OF-CHOICE strategy and shows that biased selection toward clients with higher local losses accelerates convergence by emphasizing impactful updates. Authors in \cite{fraboni2022general} develop a general theory for client sampling, linking convergence speed to aggregation weight variance and covariance, highlighting the importance of minimizing sampling variance for optimal convergence, particularly under non-IID data conditions.

Building on these insights, our work introduces an Age of Information (AoI)-based client selection mechanism using a Markov scheduling model to balance client participation. Our convergence analysis aligns with Fraboni et al.’s findings on variance minimization, and our simulations show that an optimal AoI-based approach outperforms both random and biased client selection strategies in heterogeneous settings. This approach reinforces the effectiveness of AoI-based scheduling for stable and efficient federated learning.

These studies collectively underscore a multifaceted approach to overcoming the challenges of FL, with each contribution enhancing the robustness, efficiency, and adaptability of the federated learning framework. Together, they pave the way for broader deployment of FL in practical applications, making it more viable across a wide range of real-world environments. 

\section{System Model}
\subsection{Problem Setting}
{\bf FedAvg algorithm \cite{mcmahan2017communication}.} Consider a group of \(n\) clients coordinated by a central server to collaboratively train a common neural model using federated learning (FL). Each participating client \(i\), \(i \in \{1,2,\dots,n\}\), has a dataset \(D_i = \{X_i, Y_i\}\), where \(X_i\) is the input vector and \(Y_i\) is the corresponding output vector. These datasets are used to parameterize local models directly through weights \(W_i\). The goal of FL is to optimize the collective  average of each client’s local loss function, formulated as:
\begin{align}
\min_{\{W_i\}_{i=1}^{n}} \sum_{i=1}^{n}\frac{|D_i|}{\sum_{j}|D_j|} L(X_i, Y_i, W_i),    
\end{align}
where \(n\) represents the number of clients participating in the training, $|\cdot|$ denote the cardinality of a set, and \(L\) denotes the loss function. 

FedAvg algorithm comprises two steps that are repeated until the model converges or a set number of rounds is completed:

(i) Local training: In Round $t$, clients use their local data \(D_i\) to update the local model parameters to \(W_i^{(t)}\). This involves calculating the gradient of the loss function \(L\), typically through batches of data to perform multiple gradient descent steps.

(ii) Global aggregation: all clients periodically send their updated parameters \(W_i^{(t)}\) to the central server. The central server aggregates these parameters using a specified algorithm, commonly by averaging: \(W_{global}^{(t)} = \frac{1}{m} \sum_{i=1}^{m} W_i^{(t)}\).  The aggregated global model parameters \(W_{global}^{(t)}\) are then sent back to all participating clients.
Each client sets its local model parameters to \(W_{global}^{(t)}\). 

{\bf Client selection and load balancing.}
Since communication usually occurs through a limited spectrum, only a subset $m$ out of $n$ total clients can update their parameters during each global aggregation round. This necessitates a strategy for \emph{client selection}. 

We aim to achieve load balancing by equalizing the number of iterations between consecutive client selections, assuming each client is selected with equal probability $\frac{m}{n}$. To that end, we define the \emph{load metric} \(X\) as the random variable for the number of rounds between subsequent selections of a client. Assume $X$ follows the same distribution for all clients. We propose the following optimization problem: 
\begin{align}
  &\min_{\mathcal{S}} \operatorname{Var}[X], \label{eq:min_var}\\
  &\text{s.t. } P(S_i^{(t)}=1)=\frac{m}{n}, i \in \{1,\dots,n\}, t \in \{1,\dots,T\}.\label{eq:constraint}
\end{align}
Here, $\mathcal{S}$ represents the set of permissible client selection policies, and $T$ is the total number of communication rounds. In addition, \(S_i^{(t)}\) takes value \(1\) if client \(i\) is selected during the $t$-th round, and takes value \(0\) otherwise.

Despite that all clients share the same selection probability and average workload over all iterations, as required in constraint \eqref{eq:constraint}, $\operatorname{Var}[X]$ focuses on the workload dynamics even when a small number of iterations is considered.

Minimizing \(\operatorname{Var}[X]\) creates predictable update intervals and promotes equitable load distribution among all clients. 
Moreover, the staleness of local updates, represented by a large $X$, may slow down the convergence of the training process; and frequent updates from any single client, or a small $X$, may disproportionately affect the model. 
Thus, \eqref{eq:min_var} can improve learning convergence and accuracy by ensuring consistent data freshness, which is verified in our simulation. 

{\bf Relation to Age of Information.}
A relevant metric is AoI, defined as the number of rounds elapsed since the last time a client was selected.

The age $A_i$ for Client \(i\) in Round $t$ evolves as follows:
\begin{align}
A_i^{\left ( t+1 \right )} =\left ( A_i^{\left ( t \right )} +1 \right ) \left ( 1-S_i^{\left ( t \right )}  \right ),  S_i^{\left ( t \right ) }\in \left \{ 0,1 \right \},
\end{align}
where \(A_i^{(0)}=0\). Therefore, in each round, a client's age increases by one if it is not selected and it resets to zero if it is selected.
The load metric $X$ is equal to the peak age, or the age before a client is selected.

\subsection{Client Selection Policy}
In our federated learning framework, clients are selected for parameter updates based on specific selection policies. The AoI for each client \(i\) is calculated, where \(i \in [1,2,...,n]\), and this information is used to guide the selection of clients for parameter updates. After the clients are selected, the global model parameters \(\widetilde{W}\) are sent to the chosen clients to update their local models using supervised learning algorithms on their private training data. Below, we describe three distinct client selection policies.

\begin{enumerate}
    \item \label{policy:random} \textbf{Random Selection with Data Size Weighting:} \\
    Similar to FedAvg, in this policy, \(m\) out of \(n\) clients are randomly selected in each round. Once selected, client \(i\) is assigned a weight \(\omega_i = \frac{d_i}{\sum_{j \in S} d_j}\), where \(S\) is the set of selected clients and \(d_i\) represents the data size of client \(i\). Clients that are not selected are assigned \(w_j = 0\). This ensures that clients with larger datasets have more influence during the aggregation process.

    \item \label{policy:probabilistic} \textbf{Probabilistic Selection Based on Dataset Size:} \\
    Clients are selected with replacement, and the probability of selecting client \(i\) is proportional to the size of its dataset \(d_i\). The probability \(q_i\) of selecting client \(i\) is given by:
    $$
    q_i = \frac{d_i}{\sum_{j=1}^{n} d_j}.
    $$
    If a client is selected \(l\) times out of \(m\) total selections, it is assigned a weight \(\omega_i = \frac{l}{m}\), and \(\omega_i = 0\) for non-selected clients. This policy introduces variability based on dataset size while maintaining fairness in the number of selections.

\item \label{policy:markov} \textbf{Markov Selection Based on Age of Clients:} \\
In this policy, clients are selected according to a Markov process based on their age, which represents the time since their last update. The age of a client can range from 0 to \(m'\), and the selection probability \(p_a\) for a client at age \(a\) is defined. The steady-state probability \(\pi_a\) represents the likelihood of a client being in age state \(a\). The weight assigned to selected clients is distributed uniformly. Specifically, for a client in the age state 0, the probability \(\pi_0 = \frac{m}{n}\) indicates the expected probability that a client will be selected in steady state. This model allows the selection probability to evolve as an increasing function of \(p_a\), favoring clients with older information over time.

\begin{enumerate}
    \item \textit{Optimal Markov Policy:} The selection probabilities \(p_a\) are optimally computed to achieve load balancing across clients. 

    \item \textit{Non-Optimal Markov Policy:} The selection probabilities \(p_a\) are defined to increase monotonically with client age, i.e., \(p_0 < p_1 < \dots < p_{m'}\). This encourages selection of clients with older updates more frequently, enhancing fairness, but without guaranteeing optimal load balancing.
\end{enumerate}

\end{enumerate}

\begin{figure}
    \centering
\scalebox{0.54}{
\begin{tikzpicture}[->, >=stealth', auto, semithick, node distance=2.5cm, font=\Large]
\tikzstyle{state}=[circle, draw, minimum size=1.43cm]

\node[state] (zero) {0};
\node[state] (one) [right of=zero] {1};
\node[state] (two) [right of=one] {$2$};
\node (dots) [right=0.6cm of two] {$\cdots$}; %
\node[state] (m'-1) [right=0.6cm of dots] {$m'-1$}; %
\node[state] (m') [right of=m'-1] {$m'$};

\path (zero) edge [loop above] node {$p_0$} (zero)
      (zero) edge [bend left=50] node [above] {$1-p_0$} (one)
      (one) edge [bend left=50] node [below] {$p_1$} (zero)
      (one) edge [bend left=50] node [above] {$1-p_1$} (two)
      (two) edge [bend left=50] node [below] {$p_2$} (zero)
      (m'-1) edge [bend left=50] node [above] {$1-p_{m'-1}$} (m')
      (m'-1) edge [bend left=50] node [below] {$p_{m'-1}$} (zero)
      (m') edge [loop above] node {$1-p_{m'}$} (m')
      (m') edge [bend left=50] node [below] {$p_{m'}$} (zero);

\end{tikzpicture}
} 
\caption{Markov chain with $m'+1$ states.}
\label{markov}
\end{figure}
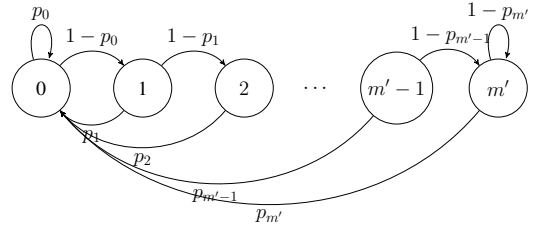

\subsection{Mathematical Analysis of Age-Based Client Selection}

In this section, we focus on a more detailed mathematical treatment of the age-based client selection policy introduced in Section \ref{policy:markov}. Specifically, we analyze the number of rounds it takes for a specific client to be selected in a Markov-based client selection system. We model this selection process using a generalized form of the geometric distribution and a Markov chain framework. The following analysis derives the expected number of rounds and the variance of the selection time for a single client.

For randomly selecting \( m \) out of \( n \) clients in each round, let the random variable \( X \) represent the number of rounds it takes for one specific client to be selected. This situation can be modeled using a generalized geometric distribution, as the process involves repeated trials until success (i.e., the specific client being selected).

The geometric distribution is defined as:
\[
P(X=x)=r(1-r)^{x-1},
\]
where \( r \) is the probability of success on any given trial.

The random variable \( X \) follows a geometric distribution with the following characteristics:
\[
\text{Mean} = \frac{1}{r} = \frac{n}{m} \quad \text{and} \quad \text{Var} = \frac{1-r}{r^2} = \frac{n(n-m)}{m^2}.
\]

Building on this, we now introduce the \textbf{Markov chain model for client selection} in the federated learning setting. Each state of the Markov chain represents the \textbf{age} of a client, i.e., the number of rounds since the client was last selected. The Markov chain consists of \( m'+1 \) states, where state \( j \in \{0, 1, ..., m'\} \) represents the client's current age, with \( m' \) being the maximum permissible age before the client remains in a state until selected.

The transitions between these states are governed by the \textbf{transition probabilities} \( p_j \), which represent the probability that a client in state \( j \) is selected. If a client is selected, it transitions to state 0. If not, the client transitions to state \( j+1 \) (if \( j < m' \)) or remains in state \( m' \) (if \( j = m' \)). These transitions are depicted in Figure \ref{markov}.

The steady-state probabilities \(\pi_a\) are determined by the following equations:
\begin{small}
\begin{align}
    \pi_0 &=  \frac{1}{1+\sum_{i=0}^{{m'}-2} \prod_{j=0}^{i}(1-p_j) + \frac{1}{p_{m'}} \prod_{j=0}^{m'-1}(1-p_j)}, \\
    \pi_i &= \frac{\prod_{j=0}^{i-1}(1-p_j)}{1+\sum_{i=0}^{m'-2} \prod_{j=0}^{i}(1-p_j) + \frac{1}{p_{m'}} \prod_{j=0}^{m'-1}(1-p_j)} \\& \quad \text{for } i \in \{1,\dots,m'-1\}, \\
    \pi_{m'} &= \frac{\frac{1}{p_{m'}}\prod_{j=0}^{m'-1}(1-p_j)}{1+\sum_{i=0}^{m'-2} \prod_{j=0}^{i}(1-p_j) + \frac{1}{p_{m'}} \prod_{j=0}^{m'-1}(1-p_j)}.
\end{align}
\end{small}
The goal of this selection process is to minimize the \textbf{variance of the load metric \( X \)}, which represents the number of rounds between consecutive selections of a client. Minimizing this variance promotes consistent client updates and balanced participation in the training process.

As per our previous work \cite{10901731}, the optimal transition probabilities are given by the following theorem:

\begin{theorem}\label{thm:optimal_general}
Consider an arbitrary \( m' \ge 1 \).
\begin{itemize}
    \item When \( m' \le \lfloor \frac{n}{m} \rfloor-1 \), the optimal values of \( p_i \) for minimizing the variance are:
    \[
    [p_0^*, p_1^*, \dots, p_{m'-1}^*, p_{m'}^*] = [0, 0, \dots, 0, \frac{1}{\frac{n}{m} - m'}],
    \]
    and the minimum variance is:
    \[
    \left( \frac{n}{m} - m' \right)\left( \frac{n}{m} - (m'+1) \right).
    \]
    
    \item When \( m' \ge \lfloor \frac{n}{m} \rfloor \), setting \( i = \lfloor \frac{n}{m} \rfloor \), the optimal values are:
    \begin{align}
    &[p_0^*, p_1^*, \dots, p_{i-2}^*, p_{i-1}^*, p_i^*, \dots, p_{m'}^*] \notag \\
    &= [0, 0, \dots, 0, i+1 - \frac{n}{m}, 1, \dots, 1],
    \end{align}
    and the minimum variance is:
    \[
    c(1 - c),
    \]
    where \( c = \frac{n}{m} - \lfloor \frac{n}{m} \rfloor \).
\end{itemize}
\end{theorem}


Having established our AoI‐based client selection framework and derived key metrics—such as the load metric and optimal transition probabilities—we now shift our focus to understanding how these selection policies influence the overall learning process. In the following section, we analyze the convergence behavior of federated averaging. This analysis examines how balanced client participation contributes to more stable global updates and accelerates convergence towards the optimal model. By linking our earlier findings on load variance and selection skew to the theoretical underpinnings of convergence, we demonstrate that the benefits of our client selection approach extend beyond equitable participation. Specifically, the analysis reveals the relationship between client selection dynamics and the efficiency of the learning process, providing deeper insights into the interplay between theory and performance.
\section{Convergence Analyses}
This section details the convergence results for federated learning under a biased client selection mechanism, considering scenarios with partial device participation. Our discussion includes an analysis of the local-global objective discrepancy, symbolized by 
$\Gamma$, alongside the impact of selection bias, represented by 
$\underline{\rho}$ and 
$\overline{\rho}$. 
Theorem \ref{th1}, introduces our study on the convergence properties of federated averaging influenced by a biased client selection approach and illustrates
the dependence of FL convergence on the statistical properties of the aggregation weights.
\section*{Assumptions and Definitions}
\begin{table*}[htbp]
    \centering
    \caption{Notations.}\label{tab:notations}
    \begin{tabular}{@{}ll@{}}
    \toprule
    Symbol & Description \\
    \midrule
    $n$ & Number of clients \\
    $K$ & Number of local SGD \\
    $m$ & Number of selected clients\\
    $m'$ & Maximum permissible age\\
    $L$ & Lipschitz smoothness parameter \\
    $\mu$ & $\mu$-strongly convex\\
    $\eta_t$ & Learning rate \\
    $\theta^t$ & Global model at server iteration $t$ \\
    $\theta^*$ & Optimum of the federated loss function \\
    $\theta_i^{t+1}$ & Local update of client $i$ on model $\theta^t$ \\
    $y_{i,k}$ & Local model of client $i$ after $k$ SGD ($y_{i,K} = \theta_i^{t+1}$ \& $y_{i,0} = \theta^t$) \\
    $q_i$ & Importance of client $i$ in the federated loss function \\
    $p_a$ & Selection probability for a client at age \(a\) in the Markov chain\\
    $S(h,t)$ & Set of participating clients given State $h$ in Iteration $t$ \\
    $\omega_i(h',t)$ & Aggregation weight for client $i$ given State $h'$ in Iteration $t$  \\
    $\omega_i$ & Aggregation weight for client $i$ \\
    $\mathbb{E}_t[\cdot]$ & Expected value conditioned on $\theta_t$ \\
    $F(\theta)$ & Global loss function \\
    $F_i(\theta)$ & Local loss function of Client $i$ \\
    $\nabla F_i(\theta)$ & Gradient of Client $i$.\\
    $g_i(\theta,\mathcal{B})$ & Stochastic gradient of Client $i$ using mini-batch $\mathcal{B}$. 
    \\
    $G^2$ & Bound on the variance of the stochastic gradients \\
    \bottomrule
    \end{tabular}
\end{table*}

 Table \ref{tab:notations} lists all the notations in the federated learning problem. 

Consider the global loss function $F(\theta)$ defined as:
\begin{align}
F(\theta) =  \sum_{i=1}^{n} q_i F_i(\theta) 
\end{align}
where $F_i(\theta) = \frac{1}{|D_i|} \sum_{x \in D_i} f(\theta; x)$ is the local objective function of client \(i\) with dataset $D_i$ and $q_i$ is the importance of client $i$ in the federated loss function.  
Let the global optimal model be $\theta^* = \arg\min_{\theta} F(\theta)$ and the
local optimum be $\theta^*_i = \arg\min_{\theta} F_i(\theta)$, $i \in \{1,\dots,n\}$.
We define $
F^* = \underset{\theta}{\min} F(\theta) = F(\theta^*)$ and $F^*_i = \underset{\theta}{\min} F
_i(\theta) = F_i(\theta^*_i)$.

Under a given policy, the set of selected clients is denoted by $S(h,t) \subseteq \{1,2, \dots,n\}$, where $h$ represents the state and $t$ is the iteration. Depending on the policy, $h$ can include the model parameters and the ages of the clients, etc. 
Moreover, the policy also determines $\omega_k(h',t)$, the aggregation weight for Client $k \in S(h,t)$, where $t$ is the iteration, and $h'$ is the state that can include the selected subset $S(h,t)$, the model parameters and the ages of the clients, etc.  
These weights ensure that contributions from different clients are properly balanced. 
Note that for fixed $t$ and $h$ (or $h'$), the variable $S(h,t)$ (or $\omega_k(h',t)$) is random due to the probabilistic nature of the policy. Moreover, for each realization of the federated learning steps, the state $h$ (or $h'$) is also random for a fixed $t$. When it is clear from the context, we omit $h,h'$ and/or $t$ in our notations.

Assume the number of selected clients is a constant $m$: 
\begin{align}
    |S|=m, 
\end{align}
and the weights satisfy the normalization condition,
\begin{equation}
    \sum_{i \in S} \omega_i =1,
\end{equation}
ensuring that the sum of weights for the selected clients $S$ equals one, and that the global model is a convex combination of local models. We define the variability and expectation of the weighting strategy by:
\begin{align}
    &\gamma_i = \mathbb{E}[\omega_i^2]= \text{Var}[\omega_i] +  \mathbb{E}[\omega_i]^2, \label{eq:gamma}\\
    &\Sigma = \sum_{i=1}^n \text{Var}[\omega_i], \label{eq:sigma}
\end{align}
where the variance and expectation are taken over the randomness of the states, $S$, and $\omega_i$. Whenever we take expectations with respect to the random policy, we need to take the maximum over all $t$, because the states $h$ and $h'$ may be of different distributions for different $t$.
Here, $\gamma_i$ captures the variance and squared expectation of the weights for client $i$, and $\Sigma$ aggregates these variances across all clients, reflecting the diversity in client contributions. 

{\bf Global aggregation.}
We introduce the global aggregation as follows:
\begin{equation}
    \theta^{t+1} = \sum_{i \in S} \omega_i \theta_i^{t+1}, \label{eq:global}
\end{equation}
where $\theta^{t+1}$ represents the global model parameters at iteration $t+1$, aggregated as a weighted sum of local model parameters $\theta_i^{t+1}$ from each client $i$, with weights $\omega_i = \omega_i(h',t)$.  

{\bf Local update.} 
In Iteration $t$, client $i$ runs $K$ epochs using random mini-batches of local data. Let $y_{i,k}^t, \mathcal{B}_{i,k}^t$ represent the local model parameters and the random mini-batch for the $k$-th epoch, respectively. 
In the first local epoch, the local model is set to be equal to the global model:
\begin{align}
    y_{i,k}^t = \theta^{t}.
\end{align}
The local gradient computation is defined by:
\begin{equation}
    d_i^{t} =\frac{1}{K} \sum_{k=0}^K g_i(y_{i,k}^t, \mathcal{B}_{i,k}^t), \label{eq:K_epochs}
\end{equation}
which is the average gradient for client $i$ at time $t$, computed over $K$ stochastic gradient epochs.  

Subsequently, local model updates are computed as:
\begin{equation}
    \theta_i^{t+1} =\theta^{t} - \eta_t K d_i^{t},\label{eq:local}
\end{equation}
indicating the adjustment of client $i$'s model parameters based on the global model parameters $\theta^{t}$, the learning rate $\eta_t$, and the average gradient $d_i^t$.

\begin{assump}\label{assumpt:1}
$F_1, \ldots, F_n$ are all $L$-smooth, i.e., for all $v$ and $u$,
\begin{align}
F_i(v) \leq F_i(u) + (v - u)^T \nabla F_i(u) + \frac{L}{2} \|v - u\|^2.
\end{align}    
\end{assump}

\begin{assump}
$F_1, \ldots, F_n$ are all $\mu$-strongly convex, i.e., for all $v$ and $u$,
\begin{align}
F_i(v) \geq F_i(u) + (v - u)^T \nabla F_i(u) + \frac{\mu}{2} \|v - u\|^2.
\end{align}    
\end{assump}

\begin{assump}\label{assumpt:3} 
For the mini-batch $\mathcal{B}_i$ uniformly sampled at random from $B_i$ from Client $i$, the resulting stochastic gradient is unbiased, that is, $\mathbb{E}[g_i(\theta_i, \mathcal{B}_i)] = \nabla F_i(\theta_i)$. Also, the variance of stochastic gradients is bounded: 
\begin{align}
\mathbb{E}[\|g_i(\theta_i, \mathcal{B}_i) - \nabla F_i(\theta_i)\|^2] \leq \sigma^2, \quad \forall i \in \{1, \ldots, n\}.
\end{align}
\end{assump}

\begin{assump}\label{assumpt:4} 
The stochastic gradient’s expected squared norm is uniformly bounded, i.e.,
\begin{align}
\mathbb{E}[\|g_i(\theta_i, \mathcal{B}_i)\|^2] \leq G^2 \quad \quad \forall i \in \{1, \ldots, n\}.
\end{align}
\end{assump}


Following similar assumptions as \cite{cho2020client} and \cite{fraboni2022general}, we introduce the following metrics: the local-global objective gap, the selection skew, and the variance of the weights, which feature
prominently in the convergence analysis presented in Theorem \ref{th1}. 
\begin{defn}[Local-Global Objective Gap and Selection Skew]\label{def1}
The Local-Global Objective Gap, \(\Gamma\), is defined as:
\begin{align}
\Gamma = F(\theta^*) - \sum_{k=1}^{n} q_k F_k(\theta_k^*),
\end{align}
where $q_k$ is the importance of client $k$ in the federated loss function, $\theta^*$ is the global optimum, and $\theta_k^*$ is the local optimum. 

The selection skew for averaging scheme \(\omega\) is defined as:
\begin{footnotesize}
\begin{equation}
    \rho(h,h',t; \theta^\prime) = \frac{\mathbb{E} \left[ \sum_{k \in S(h,t)} \omega_k(h',t)(F_k(\theta^\prime) - F_k(\theta_k^* )) \right] }{F(\theta^\prime) - \sum_{k=1}^{n} q_k F_k(\theta_k^*)}, \label{eq:rho}
\end{equation}
\end{footnotesize}
with
\begin{equation}\label{min}
    \underline{\rho} = \min_{h,h',t, \theta^\prime} \rho(h,h',t; \theta^\prime),
\end{equation}
\begin{equation}\label{max}  
    \overline{\rho} = \max_{h,h',t} \rho(h,h',t; \theta^*),
\end{equation}

where \(\omega_k(h',t)\) are the weights that vary per iteration, and the expectation is taken over the random selections and weights for fixed $t,h'$.
It follows from these definitions that \(\overline{\rho} \geq \underline{\rho}\), and $\rho \geq 0$.
\end{defn}

The Local-Global Objective Gap, \(\Gamma\), quantifies the maximum expected discrepancy between the global objective \(F\) and the weighted sum of local objectives \(F_k\) across all iterations. 
The selection skew, \(\rho(h,h',t; \theta^\prime)\), measures the ratio of the expected weighted sum of discrepancies between local objectives \(F_k\) and their optimal values \(F_k(\theta_k^* )\) to the discrepancy between the global objective \(F\) and the weighted sum of local objectives. In equation \eqref{eq:rho}, the numerator takes the expectation over the random choices of \(S(h,t)\) and \(\omega_{k}(h',t)\) based on the policy, for fixed \(h,h',t\).

\begin{theorem}\label{th1}
 Under Assumptions \ref{assumpt:1} and \ref{assumpt:4}, considering a decaying learning rate $\eta_t = \frac{1}{\mu(t+\gamma)}$ where $\gamma = \frac{4K(K+1)L}{\mu}$, and any client selection scheme, the error after $T$ iterations of federated averaging with partial device participation is bounded as follows:
\begin{align}
&\mathbb{E}[F(\theta(t))] - F^* \leq \frac{1}{T+\gamma} \frac{L}{2} \Bigg[ \gamma \|\theta^0 - \theta^*\|^2 \notag \\
&\quad + \frac{1}{\underline{\rho}(K-1)\mu^2} \Big( 16 G^2 K m^2 (K+1)( \Sigma + 1 ) + m K \sigma^2 \Big) \notag \\
&\quad + \frac{6L \Gamma}{(K-1)\mu^2} \Bigg] + \frac{KL}{(K-1)\mu} \Gamma \left(\frac{\overline{\rho}}{\underline{\rho}} - 1\right)
\end{align}

\end{theorem}
\textbf{Remark 1.} A critical observation from Theorem \ref{th1} is that the convergence bound is influenced by the clients' aggregation weights through the quantity $\sum_{i=1}^{n} \text{Var}[\omega_i(S_t)]$. Minimizing these terms, which are non-negative, leads to a smaller $\Sigma$, thereby resulting in faster convergence of the algorithm.

\textbf{Remark 2.} Theorem \ref{th1} also highlights that a larger selection skew, denoted by $\underline{\rho}$, contributes to faster convergence. This is a conservative estimate, since $\underline{\rho}$ is derived from the minimum selection skew $\rho$. Given that $\rho$ varies during training based on the current global model $\theta$ and local models $\theta_k$, the actual convergence rate can be improved by a factor that is at least as large as $\underline{\rho}$.

\textbf{Remark 3.} The term $\frac{KL}{(K-1)\mu } \Gamma \left(\frac{\overline{\rho}}{\underline{\rho}}-1\right)$ represents the bias, which depends on the selection strategy. By the definitions of $\underline{\rho}$ and $\overline{\rho}$, it follows that $\overline{\rho} \geq \underline{\rho}$, implying the existence of a non-zero bias. However, simulation results indicate that even with biased selection strategies, this term can be close to zero, thereby having a negligible impact on the final error floor.

In Theorem \ref{th1}, two critical terms $\Sigma$ and $\rho$ are influenced by the client selection strategy, particularly by the selection policy and the Age of Information (AoI) metric. These factors directly affect the speed at which the federated learning algorithm converges.

The term $\Sigma$ represents the variability in the weights assigned to each client, which are determined by the selection policy. Policies that introduce greater randomness, such as \textit{random selection policies}, result in greater variance in the client weights. This larger $\Sigma$ increases the variability of local model updates, negatively impacting the convergence rate. Specifically, the term $16 G^2 K m^2 (K+1)( \Sigma + 1 )$ in the theorem shows that as $\Sigma$ increases, the upper bound on the error grows, leading to slower convergence. On the other hand, \textit{AoI-based policies} aim to minimize $\Sigma$ by selecting clients with fresher updates more frequently, thus ensuring that their influence on the global model remains stable and reducing the variability in the weights. This reduction in $\Sigma$ improves convergence, as the model updates become more consistent across rounds.

Furthermore, the \textit{ selection skew} $\rho$, defined as the ratio of the expected discrepancy between local and global objectives, affects the convergence behavior. Theorem \ref{th1} introduces the term $\underline{\rho}$, which captures the \textit{minimum selection skew} across all rounds. The term $\frac{1}{\underline{\rho}(K-1)\mu^2}$ in the error bound illustrates that a smaller $\underline{\rho}$ (more imbalance in client selection) leads to a slower convergence rate. \textit{Random selection policies} can lead to smaller $\underline{\rho}$, as they may favor some clients more frequently in a given time frame. In contrast, \textit{AoI-based policies} increase $\underline{\rho}$ by ensuring that clients with more up-to-date information are chosen more regularly. This leads to a more balanced selection process over time, thereby reducing the selection skew and improving convergence. To empirically estimate the selection skew $\rho$, we approximate $\underline{\rho}$ and $\overline{\rho}$ through simulations. Specifically, we train each client individually with a large number of SGD iterations to obtain their local optimal models $\theta_k^*$. The final global model from our experiments serves as an estimate of the global optimum $\theta^*$. At each iteration $t$, we compute the selection skew $\rho(S^t, \theta^t)$ using the current global model $\theta^t$, and $\rho(S^t, \theta^*)$ using the estimated global optimum $\theta^*$. We then record the minimum and maximum values over all iterations, denoted as $\min_t \rho(S^t, \theta^t)$ and $\max_t \rho(S^t, \theta^*)$, respectively. The results for the four policies investigated in this paper are summarized in Table~\ref{table2}. We observe that the Markov policy has a larger $\underline{\rho}$ in most scenarios and therefore, is expected to converge faster roughly.

It is important to note that while the selection skew $\rho$ influences convergence behavior, its value alone does not directly determine the superiority of a client selection policy. The convergence rate, as indicated by Theorem~\ref{th1}, depends on multiple parameters. Therefore, a policy with a higher estimated $\underline{\rho}$ may not necessarily mean the convergence is faster if other factors adversely affect performance. The interplay of these parameters ultimately dictates the effectiveness of a policy, and thus $\rho$ should be considered alongside other metrics when evaluating overall performance.

The effect of \textit{Age of Information (AoI)} is particularly important in the context of federated learning with partial participation. When AoI is considered in the selection policy, clients with lower AoI (fresher updates) are chosen more frequently, reducing the skew in client selection and lowering $\Sigma$. This selection strategy not only reduces the variance in the client weights but also increases $\underline{\rho}$, as the selection becomes more balanced across clients. Theorems like this highlight that careful selection policies based on AoI can accelerate convergence by minimizing both $\Sigma$ and $\frac{1}{\underline{\rho}}$, improving the overall performance of the federated learning algorithm.


In the following, we calculate $\Sigma$ for three different selection policies: \textit{random selection}, \textit{data-size-based selection}, and \textit{AoI-based selection}. For each policy, $\Sigma$ represents the sum variance of the client weights $\omega_i$, which reflects the variability in client participation across rounds. 

\begin{theorem}\label{th3}
Consider a federated learning environment with $n$ clients, where in each round exactly $m$ clients are selected \emph{uniformly at random} as per policy \ref{policy:random}. For each client $i$, let $d_i>0$ be its data size. Let $S\subseteq \{1,\ldots,n\}$ be the randomly chosen subset of size $m$. Define the weight of client $i$ by
\begin{equation}
\label{eq:weight-datasize}
\omega_i \;=\;
\begin{cases}
\dfrac{d_i}{\sum_{j \in S} d_j}, & \text{if } i \in S,\\[1em]
0, & \text{otherwise.}
\end{cases}
\end{equation}
Then the sum of the variances of the weights is given by
\begin{align}
\Sigma
=&\sum_{i=1}^n \mathrm{Var}[\omega_i]\\
=&\;
  \frac{1}{\displaystyle \binom{n}{m}}
  \sum_{\substack{S \subseteq \{1,\dots,n\}\\|S|=m}}
  \;\sum_{i \in S}
  \frac{d_i^2}{\Bigl(\sum_{j \in S} d_j\Bigr)^2}
\;\;\\ \;\;
  -&\sum_{i=1}^n
  \Biggl(
    \frac{1}{\binom{n}{m}}
    \sum_{\substack{S: \, i \in S}}
    \frac{d_i}{\sum_{j \in S} d_j}
  \Biggr)^2.
\end{align}
\end{theorem}

\textbf{Remark 4.}\label{remark:homogeneous}
In the special case where the data is homogeneous across all clients, i.e., $d_i = d$ for all $i$, the weight for each selected client simplifies to
\[
\omega_i \;=\; \frac{d_i}{\sum_{j \in S} d_j} \;=\; \frac{d}{m \cdot d} \;=\; \frac{1}{m}.
\]
Hence, the weight distribution is uniform among the $m$ selected clients. Consequently, in this scenario:
\begin{small}
\[
\mathbb{E}[\omega_i] \;=\; \frac{m}{n} \times \frac{1}{m} \;=\; \frac{1}{n},
\quad
\sum_{i=1}^n \omega_i^2 \;=\; \sum_{i \in S} \frac{1}{m^2} \;=\; \frac{m}{m^2} \;=\; \frac{1}{m}.
\]
\end{small}
Taking expectations,
\begin{small}
\[
\mathbb{E}\Bigl[\sum_{i=1}^n \omega_i^2\Bigr] \;=\; \frac{1}{m},
\quad
\sum_{i=1}^n \mathrm{Var}[\omega_i]
\;=\;
\frac{1}{m}
\;-\;
\sum_{i=1}^n \Bigl(\frac{1}{n}\Bigr)^2
\;=\;
\frac{1}{m} - \frac{1}{n}.
\]
\end{small}
\begin{theorem} \label{th4}
Consider a federated learning environment with \(n\) clients, where in each round, \(m\) clients are selected with replacement. The probability $q_i$ of choosing client \(i\) is proportional to the size of its data set, indicated by \(d_i\), according to policy \ref{policy:probabilistic}. The expected sum of the variances of the weights for the selected clients is:
$$
\Sigma=\sum_{i=1}^{n} \text{Var}[\omega_i]= \sum_{i=1}^n \frac{q_i(1 - q_i)}{m}.
$$
\end{theorem}

\textbf{Remark 5.}\label{remark5}
In the special case where all datasets are of equal size, i.e., \(d_i = d\) for each \(i\), the selection probabilities become uniform across all \(n\) clients:
\[
q_i = \frac{d_i}{\sum_{j=1}^{n} d_j} = \frac{d}{n \, d} = \frac{1}{n}.
\]
Then each client is equally likely to be chosen in each of the \(m\) selections. Substituting \(q_i = \tfrac{1}{n}\) into
\[
\Sigma = \sum_{i=1}^{n} \text{Var}[\omega_i] = \sum_{i=1}^{n} \frac{q_i(1 - q_i)}{m},
\]
we obtain
\[
\Sigma = \sum_{i=1}^{n} \frac{\tfrac{1}{n}\bigl(1 - \tfrac{1}{n}\bigr)}{m}
= \frac{1}{m} \Bigl(1 - \tfrac{1}{n}\Bigr).
\]
\begin{theorem}\label{th5}
Consider a federated learning environment with $n$ clients. In each round, each client $i$ is in an \emph{age state} $a \in \{0,1,\dots,m'\}$ and is selected with probability $p_{a}$, according to \emph{Policy~\ref{policy:markov}} (Markov selection based on age). Denote by $S \subseteq \{1,\ldots,n\}$ the (random) set of selected clients in that round, assuming always $|S|>0$. Assign weights to the clients as
\begin{equation}
\label{eq:weight-markov}
\omega_i \;=\;
\begin{cases}
\dfrac{1}{\sum_{j \in S} 1}\;=\;\dfrac{1}{\lvert S\rvert}, & \text{if } i \in S,\\[6pt]
0, & \text{otherwise}.
\end{cases}
\end{equation}
Where ages form a Markov chain with a unique stationary distribution 
$\bigl(\pi_0,\pi_1,\dots,\pi_{m'}\bigr)$ and \emph{steady-state} selection probability of a client is
\[
   p_{\mathrm{avg}} 
   \;=\;
   \sum_{a=0}^{m'} \pi_a \,p_a.
\]
Then, the sum of the variances of the weights over all $n$ clients is given by
\begin{align}
\Sigma=
\sum_{i=1}^n \mathrm{Var}[\omega_i]
=\;
\sum_{s=1}^{n}
\frac{1}{s}
\,\binom{n}{s}\,\bigl(p_{\mathrm{avg}}\bigr)^{s}
\bigl(1 - p_{\mathrm{avg}}\bigr)^{n-s}
\;-\;
\frac{1}{n},\
\label{eq:markov-var-result}
\end{align}
\end{theorem}

\section{Numerical Results}
In this section, we evaluate the model's load distribution, convergence, and accuracy in scenarios where only a subset of clients $(15\%)$ participate in each communication round. We compare four client selection methods: random selection, probabilistic selection based on dataset size, our proposed decentralized policy using the optimal Markov model specifically designed for load balancing, and a Markov model with non-optimal $p_a$ values for load balancing that increases selection probabilities based on each client's age. Experiments are conducted on the MNIST, CIFAR-$10$, and CIFAR-$100$ datasets, using a convolutional neural network (CNN) architecture as in \cite{mcmahan2017communication} for sample classification. For local optimization, we employ stochastic gradient descent (SGD) with a batch size of $50$, running $5$ local epochs per round with an initial learning rate of $0.1$ and a decay rate of $0.998$. The experimental setup includes $100$ devices, $15$ selected per round, and a maximum client age of $10$ in the Markov model. This configuration enables a thorough examination of both optimal and non-optimal selection strategies in terms of load balancing and their effects on FL convergence, allowing for comparison with other models as well.

As shown in Figure \ref{fig:sigma_variance_over_time}, the total variance \(\Sigma\) over $1000$ rounds is compared for four different client selection policies in federated learning. The experiment involves $100$ clients, with $15$ selected per round, and non-uniform dataset sizes. For Policies $3$ and $4$ (Markov-based), we consider steady-state behavior for a fair comparison, as the variance stabilizes after a few rounds. Both Markov-based policies exhibit significantly lower variance compared to random selection (Policy $1$) and probabilistic selection (Policy $2$). Policy $3$ (optimal Markov) achieves the lowest variance (\(0.055\)), indicating the most balanced client selection. Policy $4$ (non-optimal Markov) also performs well with an average variance of \(0.061\). In contrast, random selection (Policy $1$) yields the highest variance (\(0.204\)), while probabilistic selection (Policy $2$) results in (\(0.116\)).

\begin{figure}[htbp]
    \centering
    \includegraphics[width=\columnwidth]{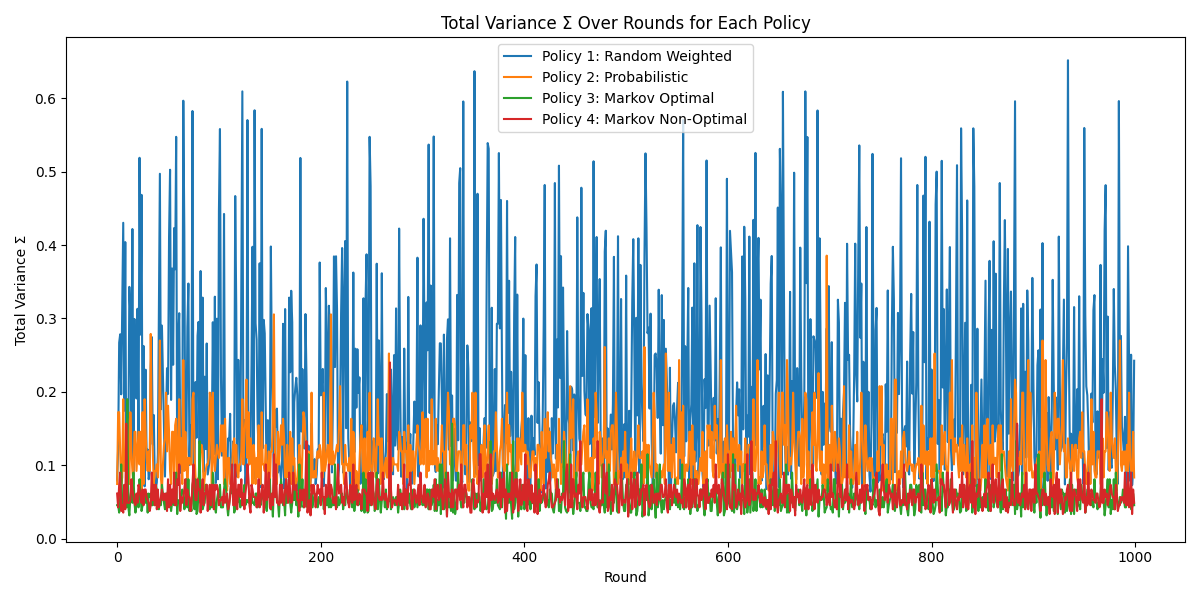}
    \caption{Total Variance \(\Sigma\) over $1000$ rounds. The simulation was performed with \(n = 100\) clients, \(k = 15\) selected clients per round, and a maximum client age of \(m = 10\). The figure compares : Policy 1 (Random Weighted), Policy $2$ (Probabilistic), Policy $3$ (Markov Optimal), and Policy $4$ (Markov Non-Optimal).}
    \label{fig:sigma_variance_over_time}
\end{figure}

\begin{figure}[htbp]
    \centering
    \includegraphics[width=\columnwidth]{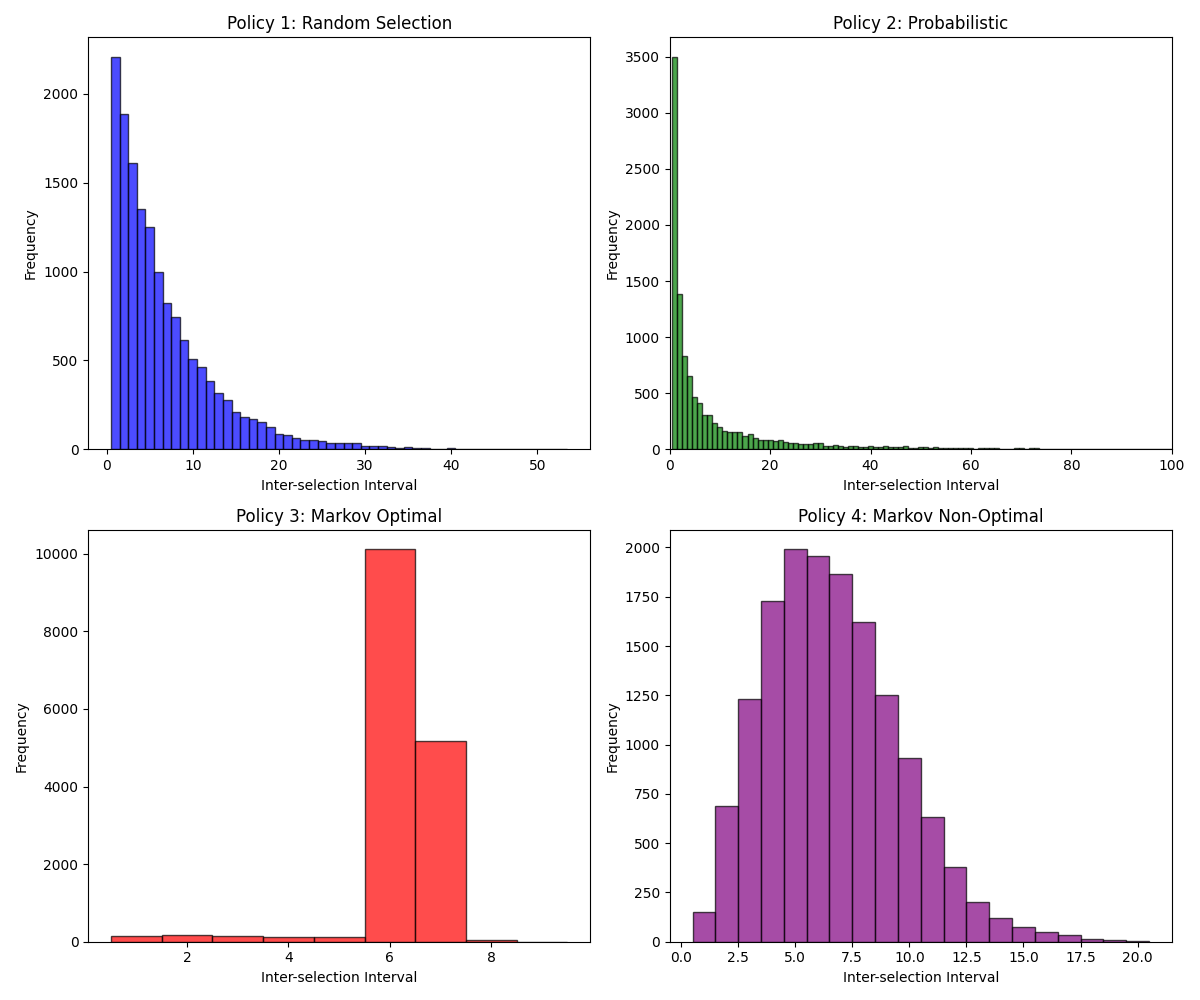}
    \caption{Empirical inter-selection interval distributions for four client selection policies over $1000$ rounds, with $100$ clients and $15$ clients selected per round.}
    \label{inter}
\end{figure}

Figure \ref{inter} presents the empirical inter-selection interval distributions across four client selection policies, reflecting how often clients are re-selected over $1000$ rounds. These distributions were obtained through simulation, and the selection process involves $100$ clients, with $15$ clients selected per round. In Policy $1$ (Random Selection), the histogram declines as the inter-selection interval increases, which is expected in a random selection process where clients are equally likely to be selected, resulting in many short inter-selection intervals and fewer long intervals. For Policy $2$ (Probabilistic Selection), where clients are selected based on their dataset size following a Zipf distribution with shape parameter $a = 2.0$, a similar steep decline is observed. However, the longer tail indicates that clients with larger datasets are selected more frequently, while clients with smaller datasets experience longer gaps between selections, leading to less balanced load distribution.

In contrast, Policy $3$ (Markov Optimal) shows a concentrated peak around $6$ rounds, demonstrating regular client selection, as this policy is optimized to prioritize clients based on their selection age. This results in better load balancing, as the selection intervals are highly consistent, ensuring that all clients are selected at regular intervals with minimal variability. Finally, Policy $4$ (Markov Non-Optimal) displays a broader distribution of intervals, indicating less consistency and regularity in selection compared to the optimal policy, though still more balanced than the random and probabilistic policies. These histograms effectively illustrate the distinct selection patterns of each policy, with Policy $3$ (Markov Optimal) offering the best load balancing, aligning with the theoretical design of these client selection methods.

The metric \( \frac{\sqrt{\text{Var}(Y)}}{T} \) assesses the stability of the client selection process, where \( Y \) represents the number of times each client is selected within a window of size \( T \). By examining the variance of \( Y \), we capture how much the selection frequency fluctuates across clients within each window. Normalizing by \( T \) allows us to account for different window sizes, enabling a fair comparison of client selection variability across time scales. A lower value of this metric indicates a more balanced and stable selection of clients.

\begin{figure}[!t]
    \centering
    \includegraphics[width=\columnwidth]{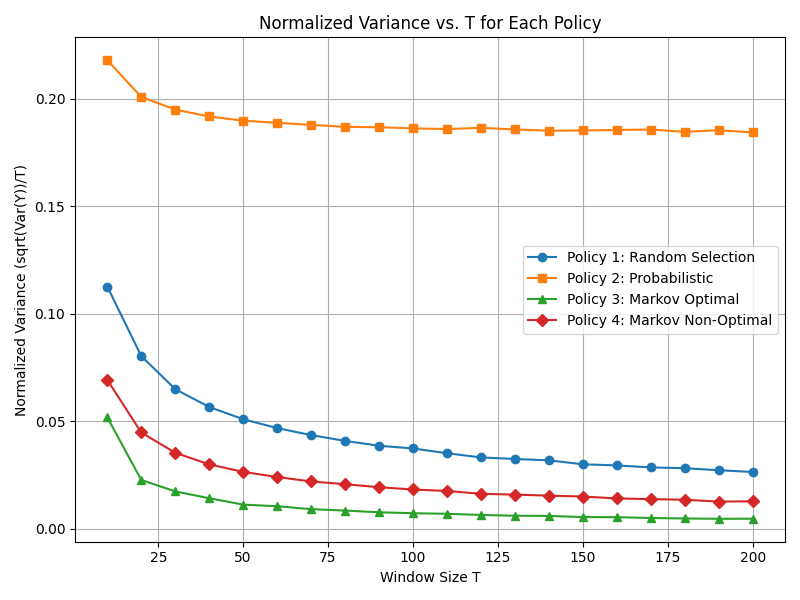}
    \caption{Normalized variance \( \frac{\sqrt{\text{Var}(Y)}}{T} \) versus window size \( T \) for four client selection policies: Random selection (Policy 1), Probabilistic (Policy $2$), Markov Optimal (Policy $3$), and Markov Non-Optimal (Policy $4$). The setup involves $100$ clients, $1000$ rounds, and the selection of $15$ clients per round.}
    \label{var_Y}
\end{figure}

In Figure \ref{var_Y}, we plot this metric across different window sizes \( T \) for four selection policies. As expected, increasing the window size \( T \) generally reduces the normalized variance for all policies, indicating more stable client selection as the time window grows. Notably, Policy $2$ exhibits the highest normalized variance across all window sizes. This reflects the inherent imbalance introduced by assigning selection probabilities based on client data sizes, leading to greater variability in client participation, even as the window size increases. Policy $1$ (Random Selection) shows a lower normalized variance than Policy $2$, but its randomness still results in relatively high variability, particularly at smaller window sizes.

Our proposed policy (Markov Optimal), demonstrates the lowest normalized variance across all window sizes, including small \( T \), indicating that it stabilizes client selection earlier and more effectively than the other methods. This suggests that even with smaller windows, our method ensures a balanced distribution of client selections, making it well-suited for scenarios that require frequent model updates or tight time constraints. Similarly, Policy $4$ (Markov Non-Optimal), while not as stable as Policy $3$, still shows relatively low variance compared to the probabilistic and random policies, benefiting from its Markov structure. 

\begin{figure}[htbp]
    \centering
    \begin{subfigure}[b]{\columnwidth}
        \centering
        \includegraphics[width=\linewidth]{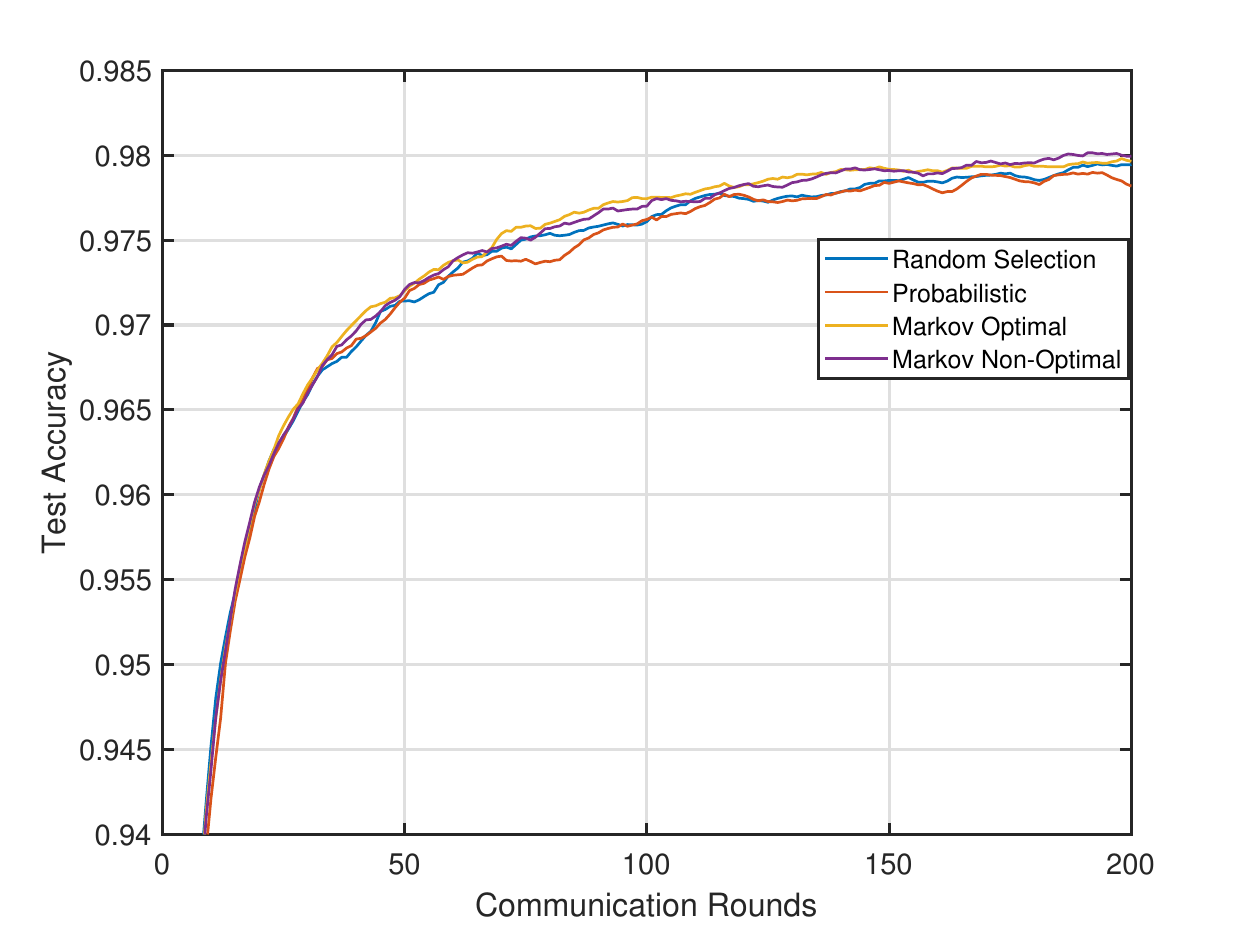}
        \caption{}
        \label{mnist_iid}
    \end{subfigure}
    \vspace{1em} 
    \begin{subfigure}[b]{\columnwidth}
        \centering
        \includegraphics[width=\linewidth]{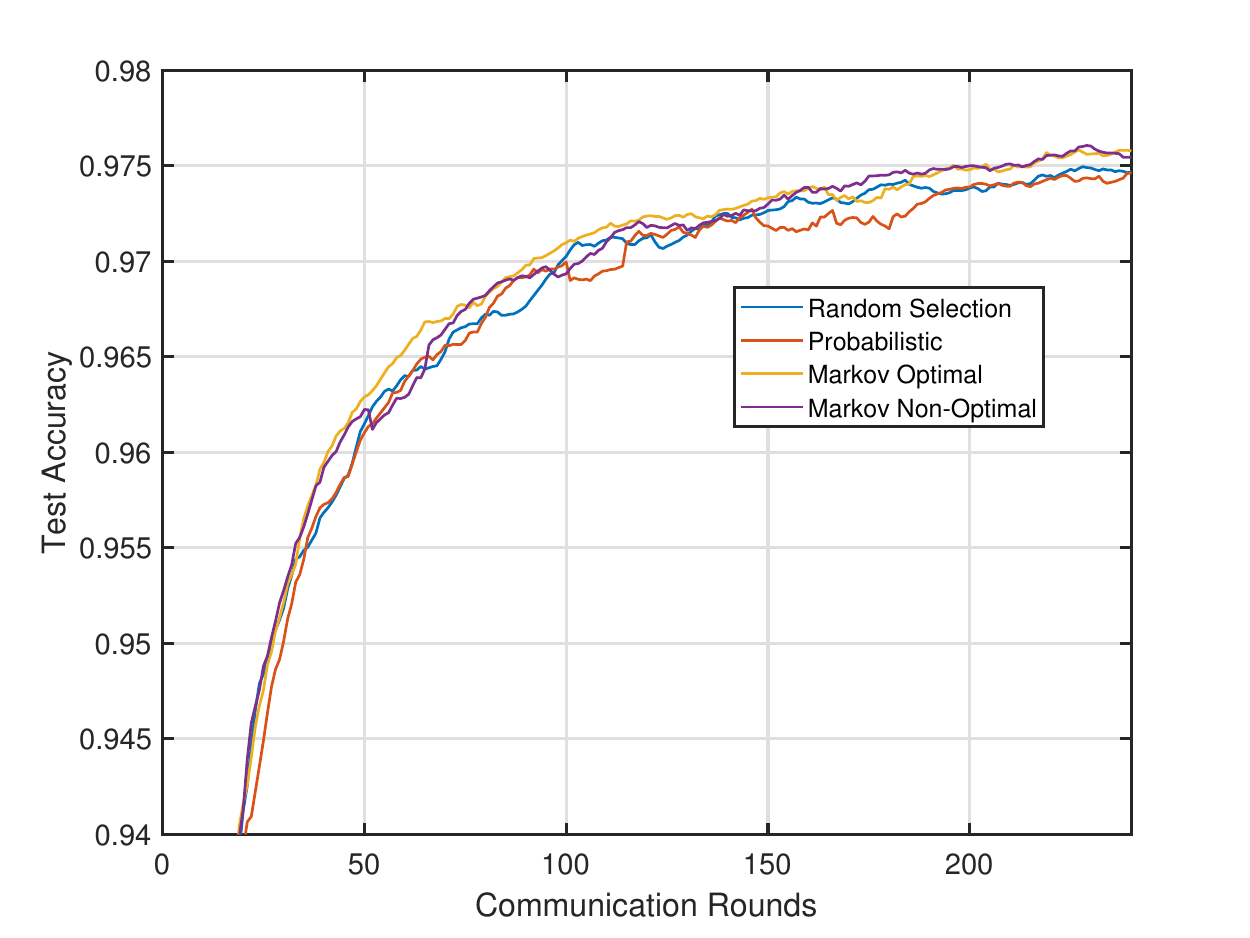}
        \caption{}
        \label{mnist_noniid}
    \end{subfigure}
    \caption{Accuracy on MNIST dataset across different client selection policies: Random selection (Policy $1$), Probabilistic (Policy $2$), Markov Optimal (Policy $3$), and Markov Non-Optimal (Policy $4$). The experiment was conducted with $100$ clients, $15$ clients selected per round, and a maximum client age of $10$. (a) IID dataset. (b) Non-IID dataset (generated using a Dirichlet distribution with $\alpha = 0.3$).}
    \label{mnist_combined}
\end{figure}

As illustrated in Figures \ref{mnist_iid} and \ref{mnist_noniid}, Markov-based client selection policies outperform the FedAvg (random selection) policy for both IID and non-IID scenarios on the MNIST dataset. Specifically, in the IID setting, the Markov optimal policy reaches $97$\% accuracy in $39$ rounds, compared to 45 rounds for random selection, achieving a $13$\% faster convergence. In the non-IID scenario, it achieves the same accuracy in $91$ rounds versus $99$ rounds for random selection, reflecting an $8$\% improvement. Additionally, the Markov optimal policy ensures more consistent and stable learning while enhancing load balancing across clients, which is crucial for maintaining system efficiency and fairness. These results illustrate efficiency and reliability of Markov-based strategies in federated learning on the MNIST dataset, with their ability to facilitate a more balanced and effective training process.

As shown in Figures \ref{cifar10_iid} and \ref{cifar10_iid_non}, Markov-based client selection policies surpass FedAvg (random selection) and probabilistic methods for both IID and non-IID CIFAR-$10$ datasets. In the IID setting, the Markov optimal policy reaches $75$\% accuracy in $95$ rounds, compared to $108$ rounds for random selection, achieving a $12$\% faster convergence. In the non-IID scenario, it attains the same accuracy in $239$ rounds versus $261$ rounds for random selection, marking a $9$\% improvement. Furthermore, the non-optimal Markov policy reaches $105$ rounds in the IID setting and $247$ rounds in the non-IID setting, both of which are faster than random and probabilistic methods but still lag behind the optimal Markov policy. The Markov optimal policy consistently demonstrates faster and more stable convergence, highlighting its effectiveness in load balancing and overall federated learning performance. 

As depicted in Figures \ref{cifar100_iid} and \ref{cifar100_iid_non}, the CIFAR-100 dataset results demonstrate faster convergence across both IID and non-IID data distributions. Specifically, for the IID setting, the Markov optimal policy achieves 30\% accuracy in $139$ rounds, outperforming the non-optimal Markov policy at $158$ rounds, random selection at $154$ rounds, and the probabilistic method at $188$ rounds, representing a $10$-$20$\% faster convergence compared to these methods. In the non-IID scenario, the Markov optimal policy reaches the same $30$\% accuracy in $87$ rounds, compared to $93$ rounds for random selection and 111 rounds for the probabilistic method, indicating a $7.5$\% improvement. The non-optimal Markov policy also performs better than random and probabilistic approaches, but does not match the efficiency of the optimal Markov policy. Additionally, the Markov optimal policy maintains steady improvement and superior load balancing, ensuring that all clients contribute effectively and equitably. These numerical results highlight the advantages of age-based client selection strategies for complex datasets like CIFAR-$100$, demonstrating that the structured design of the Markov optimal policy leads to the most effective and reliable learning performance in federated learning environments.

\begin{figure}[htbp]
    \centering
    \begin{subfigure}[b]{\columnwidth}
        \centering
        \includegraphics[width=\linewidth]{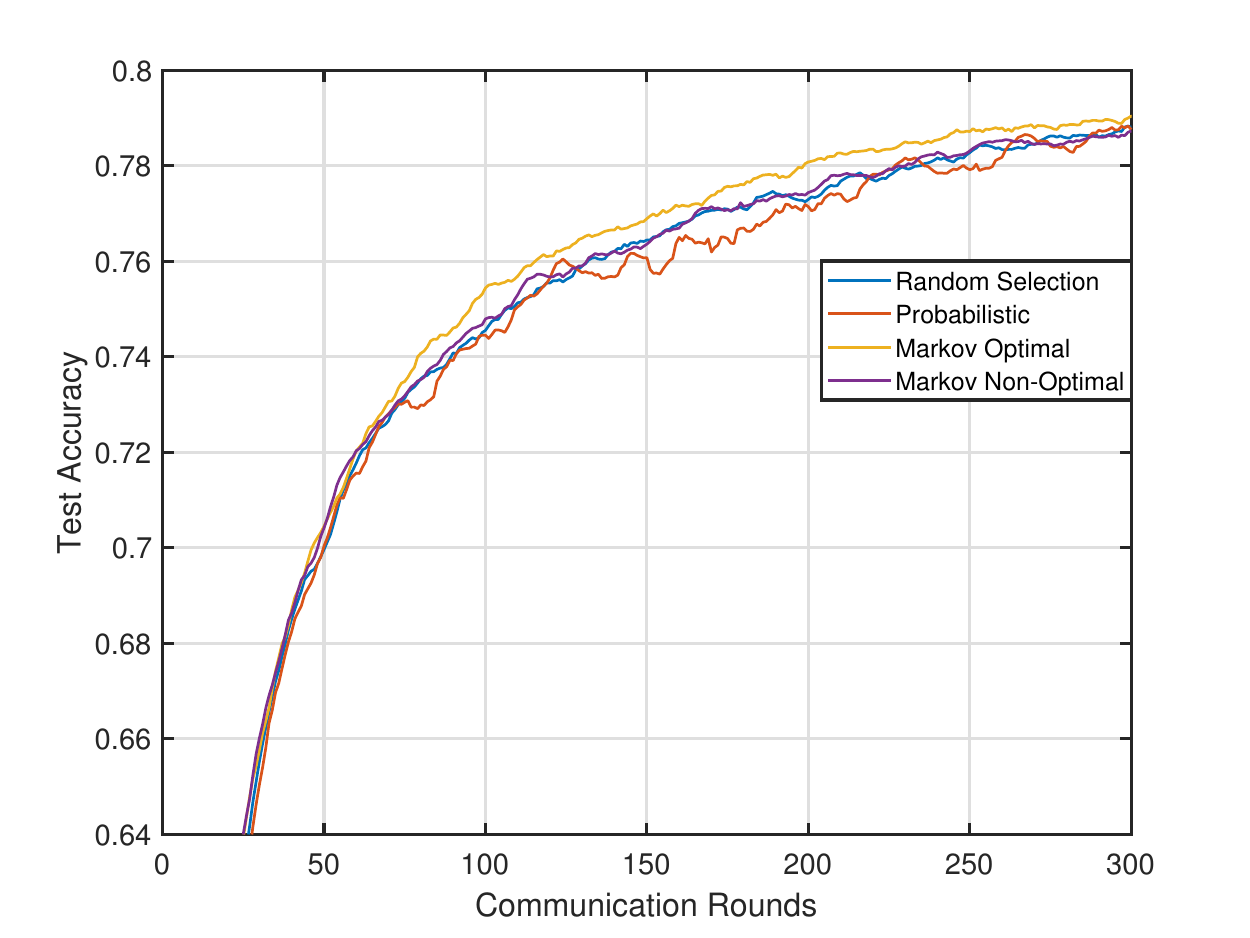}
        \caption{}
        \label{cifar10_iid}
    \end{subfigure}
    \vspace{1em} 
    \begin{subfigure}[b]{\columnwidth}
        \centering
        \includegraphics[width=\linewidth]{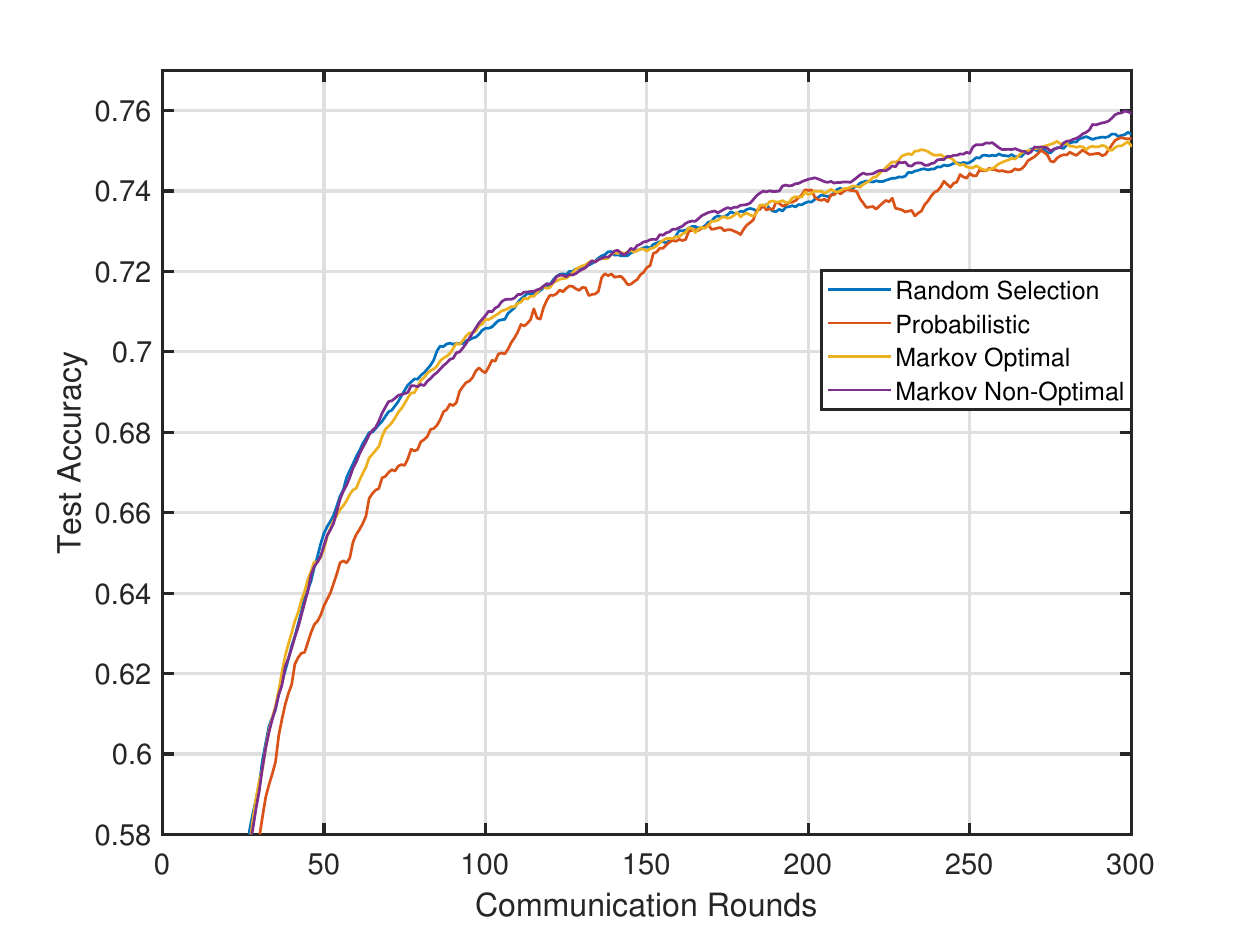}
        \caption{}
        \label{cifar10_iid_non}
    \end{subfigure}
    \caption{Accuracy on CIFAR-$10$ dataset across different client selection policies: Random selection (Policy $1$), Probabilistic (Policy $2$), Markov Optimal (Policy $3$), and Markov Non-Optimal (Policy $4$). The experiment was conducted with $100$ clients, $15$ clients selected per round, and a maximum client age of $10$. (a) IID dataset. (b) Non-IID dataset.}
    \label{cifar10_combined}
\end{figure}

\begin{figure}[htbp]
    \centering
    \begin{subfigure}[b]{\columnwidth}
        \centering
        \includegraphics[width=\linewidth]{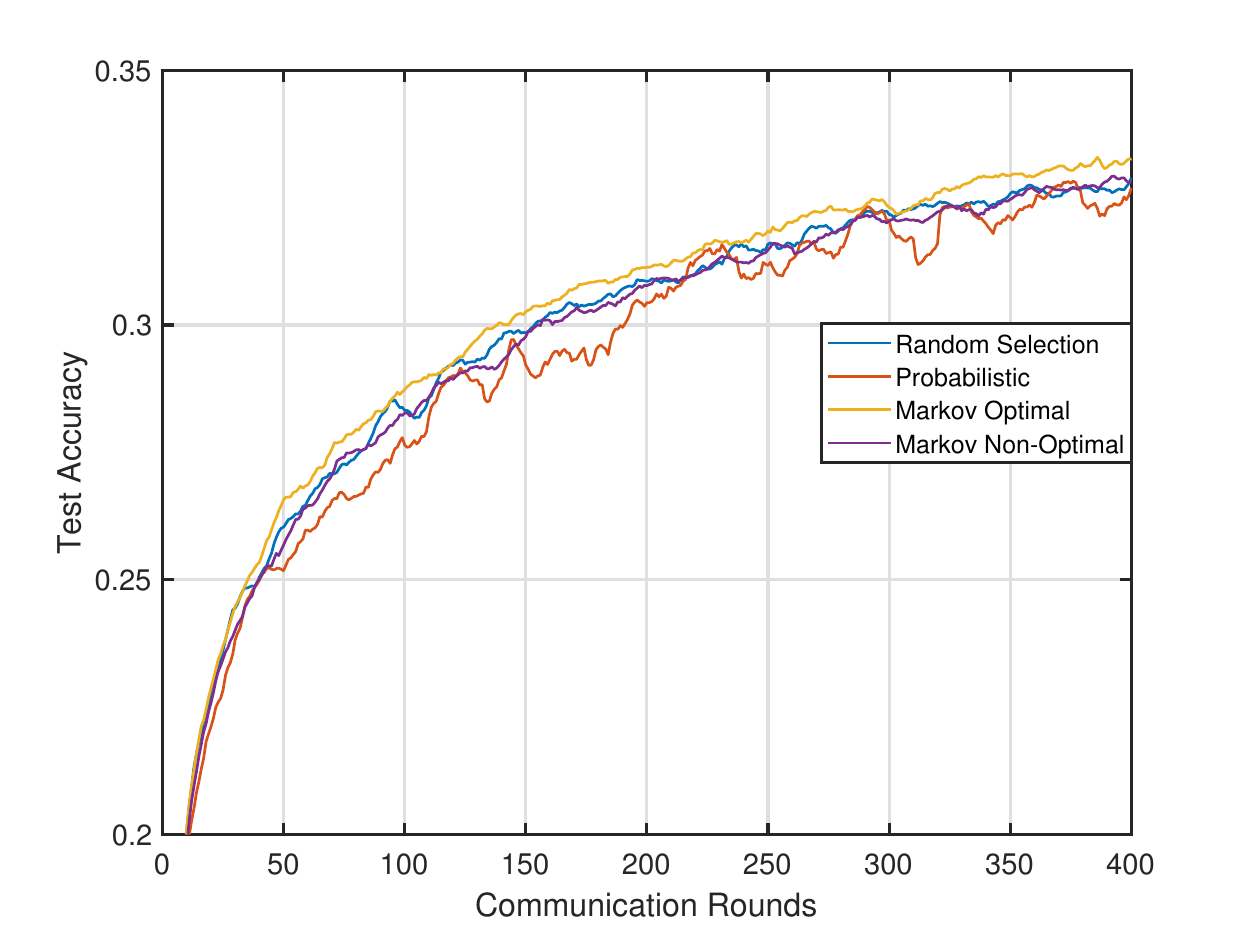}
        \caption{}
        \label{cifar100_iid}
    \end{subfigure}
    \vspace{1em} 
    \begin{subfigure}[b]{\columnwidth}
        \centering
        \includegraphics[width=\linewidth]{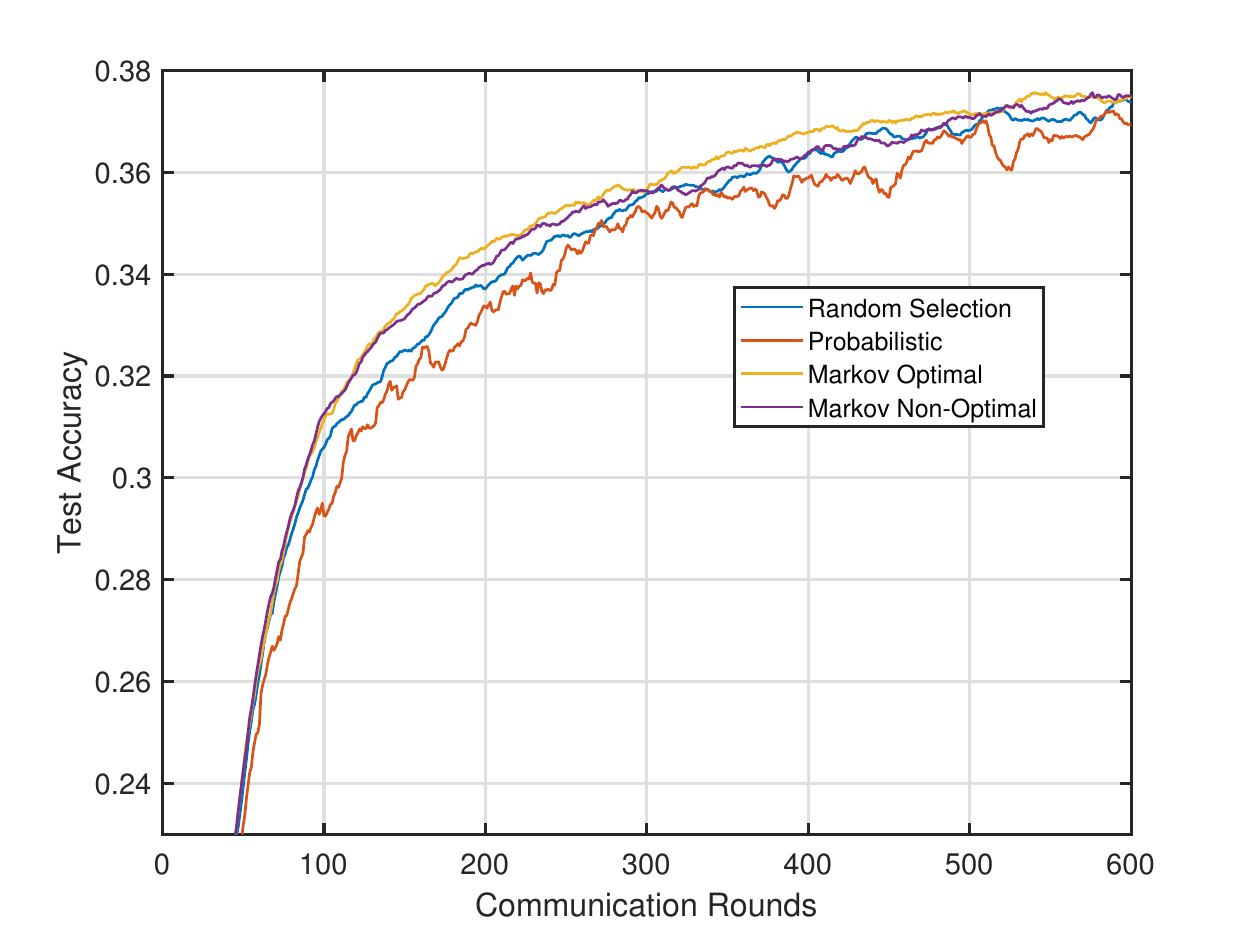}
        \caption{}
        \label{cifar100_iid_non}
    \end{subfigure}
    \caption{Accuracy on CIFAR-$100$ dataset across different client selection policies: Random selection (Policy $1$), Probabilistic (Policy $2$), Markov Optimal (Policy $3$), and Markov Non-Optimal (Policy $4$). The experiment was conducted with $100$ clients, $15$ clients selected per round, and a maximum client age of $10$. (a) IID dataset. (b) Non-IID dataset.}
    \label{cifar100_combined}
\end{figure}

\begin{table*}[t]
    \centering
    \caption{Average empirical $\underline{\rho}$ for different selection policies and datasets.}
    \label{table2}
    \begin{tabular}{|l|c|c|c|}
        \hline
        \textbf{Selection Policy} & \textbf{MNIST} & \textbf{CIFAR10} & \textbf{CIFAR100} \\
        \hline
        Random Selection & 0.98 & 0.96 & 1.33 \\
        \hline
        Probabilistic & 0.83 & 0.85 & 0.93 \\
        \hline
        Markov Optimal & 0.98 & 1.5 &1.52  \\
        \hline
        Markov Non-Optimal  & 0.99 & 1.27 & 1.41 \\
        
        \hline
    \end{tabular}
\end{table*}

\section{Conclusions}
This paper presents a novel Age of Information (AoI)-based client selection mechanism designed to directly address load imbalance and slow convergence. By employing a decentralized Markov scheduling policy, the proposed approach ensures balanced client participation and minimizes the variance in selection intervals. The theoretical framework offers a rigorous convergence proof for the AoI-based selection policy, establishing that it guarantees stable and reliable model convergence. Extensive simulations on datasets such as MNIST, CIFAR-10, and CIFAR-100 revealed that the optimal Markov policy outperformed traditional selection strategies, achieving faster convergence across both IID and non-IID scenarios, with improvements ranging from $7.5$\% to $20$\%. Additionally, the decentralized nature of the proposed method supports scalable implementations with minimal reliance on central coordination, making it well-suited for diverse and dynamic FL environments. This work highlights the potential of AoI as a useful metric for optimizing client selection in FL, paving the way for more efficient and fair collaborative learning systems.
\appendix \label{apen}
\begin{lemma} \label{lem1}
Suppose $F_i:\mathbb{R}^d\to\mathbb{R}$ is $L$-smooth with its global minimum at $\theta_i^*$.  Then for any $\theta_i$ in its domain,
\begin{equation}
    \lVert \nabla F_i(\theta_i)\rVert^2 
    \;\le\;2L\,\bigl(F_i(\theta_i)-F_i(\theta_i^*)\bigr).
    \label{eqn:lemma}
\end{equation}
\end{lemma}

\begin{proof}  (\textbf{Lemma \ref{lem1}}) \\
Defining the auxiliary point
\[
\theta_i' \;=\;\theta_i \;-\;\frac{1}{L}\,\nabla F_i(\theta_i).
\]
By $L$-smoothness,
\begin{equation}
\begin{split}
F_i(\theta_i')
&\le F_i(\theta_i)
  +\nabla F_i(\theta_i)^{T}(\theta_i'-\theta_i)\\
&\quad+\frac{L}{2}\,\|\theta_i'-\theta_i\|^2.
\end{split}
\end{equation}
Substituting $\theta_i'-\theta_i=-\tfrac{1}{L}\nabla F_i(\theta_i)$:
\begin{equation}
\begin{split}
F_i(\theta_i')
&\le F_i(\theta_i)
  -\frac{1}{L}\,\|\nabla F_i(\theta_i)\|^2\\
&\quad+\frac{1}{2L}\,\|\nabla F_i(\theta_i)\|^2\\
&=F_i(\theta_i)
  -\frac{1}{2L}\,\|\nabla F_i(\theta_i)\|^2.
\end{split}
\end{equation}
With the global minimumat $\theta_i^*$ we have, $F_i(\theta_i^*)\le F_i(\theta_i')$.  Hence
\begin{equation}
\begin{aligned}
F_i(\theta_i^*)
&\le F_i(\theta_i)
  -\frac{1}{2L}\,\|\nabla F_i(\theta_i)\|^2,\\
\|\nabla F_i(\theta_i)\|^2
&\le 2L\,\bigl(F_i(\theta_i)-F_i(\theta_i^*)\bigr).
\end{aligned}
\end{equation}
\end{proof}

\begin{lemma} \label{lem2}
 Let us consider $n$ vectors $x_1, \ldots, x_n$ and a policy with client sampling set $S$ and weight assignment  $\omega_i$, $i \in \{1,2,\dots,n\}$. We have:
 \begin{align}
   \frac{1}{m} \mathbb{E} \Bigl\| \sum_{i \in S} \omega_i x_i \Bigr\|^2
   \;\le\;
   \sum_{i \in S} \gamma_i \,\|x_i\|^2,
   \label{eq17}
 \end{align}
 where the expectation is over the random set $S$ and the random weights $\omega_i$, and $\gamma_i$ is defined in \eqref{eq:gamma}.
\end{lemma}

\begin{proof}  (\textbf{Lemma \ref{lem2}}) \\
Denote the Euclidean norm of $x_i$ by $\|x_i\|$. By Cauchy–Schwarz and $\omega_i\ge0$,
\[
\Bigl\|\sum_{i\in S}\omega_i\,x_i\Bigr\|^2
\;\le\;
\Bigl(1+\dots+1)
\Bigl(\sum_{i\in S}\omega_i^2\,\|x_i\|^2\Bigr)
\;=
m\;\sum_{i\in S}\omega_i^2\,\|x_i\|^2.
\]
Taking expectation over both $S$ and the $\omega_i$, and using
$\mathbb{E}[\omega_i^2]=\gamma_i$, we get
\[
\mathbb{E}\Bigl\|\sum_{i\in S}\omega_i\,x_i\Bigr\|^2
\;\le\;
m\;\sum_{i\in S}\gamma_i\,\|x_i\|^2.
\]
Dividing by $m$ results in \eqref{eq17}.
\end{proof}

\begin{lemma} \label{lem3}
The expected average discrepancy between $\theta^{t}$ and $\theta_k^{t}$ for $k \in S$ satisfies
\[
\frac{1}{m}
\mathbb{E}\Bigl[\sum_{k \in S}\|\theta^t - \theta_k^t\|^2\Bigr]
\;\le\;
16\,G^2\,\eta_t^2\,K^2\,m\,(\Sigma+1).
\]
\end{lemma}

\begin{proof} (\textbf{Lemma \ref{lem3}}) \\
From \eqref{eq:local} and \eqref{eq:global} we have for each $k\in S$,
\[
\theta^t - \theta_k^t
=\sum_{i\in S}\omega_i(\theta_i^t - \theta_k^t)
=\eta_{t-1}\,K\sum_{i\in S}\omega_i\,(d_i^{\,t-1}-d_k^{\,t-1}).
\]
Hence
\[
\frac{1}{m}\sum_{k\in S}\|\theta^t - \theta_k^t\|^2
=(\eta_{t-1}K)^2
\;\frac{1}{m}\sum_{k\in S}
\Bigl\|\sum_{i\in S}\omega_i(d_i^{\,t-1}-d_k^{\,t-1})\Bigr\|^2.
\]
Taking expectation and applying Lemma~\ref{lem2} with $x_i=d_i^{\,t-1}-d_k^{\,t-1}$ yields
\[
\frac{1}{m}\mathbb{E}\Bigl[\sum_{k\in S}\|\theta^t - \theta_k^t\|^2\Bigr]
\le
(\eta_{t-1}K)^2
\sum_{k\in S}
\;\sum_{i \in S} \gamma_i 
 \mathbb{E}\|d_i^{\,t-1}-d_k^{\,t-1}\|^2.
\]
By \eqref{eq:K_epochs} and the bound $\mathbb{E}\|g_i\|^2\le G^2$, we have
$\mathbb{E}\|d_i^{\,t-1}\|^2\le G^2$, and thus
\[
\mathbb{E}\|d_i^{\,t-1}-d_k^{\,t-1}\|^2
\;\le\;
2\,\mathbb{E}\|d_i^{\,t-1}\|^2
+2\,\mathbb{E}\|d_k^{\,t-1}\|^2
\;\le\;4G^2.
\]
Substituting gives
\begin{align*}
\frac{1}{m}\,\mathbb{E}\Bigl[\sum_{k\in S}\|\theta^t - \theta_k^t\|^2\Bigr]
&\le 
(\eta_{t-1}K)^2
\sum_{k\in S}\sum_{i\in S}\gamma_i\;4G^2
\\
&=
4G^2\,\eta_{t-1}^2\,K^2
\sum_{k\in S}\sum_{i\in S}\gamma_i
\\
&\le
4G^2\,\eta_{t-1}^2\,K^2\,m(\Sigma+1).
\end{align*}

Finally, using $\eta_{t-1}\le2\eta_t$ and $\displaystyle\sum_{i \in S} \gamma_i
 \le \Sigma+1$, gives us the claimed bound.
\end{proof}

\begin{lemma} \label{lem4}
With $\mathbb{E}[\cdot]$, the total expectation over all random sources including the random source from the selection strategy, we have the upper bound:
\[
\mathbb{E}[||\theta^{t} - \theta^*||^2] \leq \mathbb{E} \left[ \sum_{k \in S} \omega_k \| \theta_k^{t} - \theta^* \|^2 \right]. \tag{24}
\]
\end{lemma} 

\begin{proof} (\textbf{Lemma \ref{lem4}}) \\
Consider the expression for the expected squared norm difference:
\begin{align*}
\mathbb{E}[||\theta^{t} - \theta^*||^2] &= \mathbb{E}\left[ \left\| \sum_{k \in S} \omega_k (\theta_k^{t} - \theta^*) \right\|^2 \right].
\end{align*}

Applying the Cauchy-Schwarz inequality, we obtain:

\begin{align*}
\left( \sum_{k \in S} \omega_k (\theta_k^{t} - \theta^*) \right)^2 \leq \left( \sum_{k \in S} \omega_k \right) \left( \sum_{k \in S} \omega_k (\theta_k^{t} - \theta^*)^2 \right).
\end{align*}

Since $\sum_{k \in S} \omega_k = 1$ and by taking the expectation, this simplifies to:

\begin{align*}
\mathbb{E}\left[ \left\| \sum_{k \in S} \omega_k (\theta_k^{t} - \theta^*) \right\|^2 \right] &\leq \mathbb{E} \left[ \sum_{k \in S} \omega_k \| \theta_k^{t} - \theta^* \|^2 \right].
\end{align*}
\end{proof}
\begin{proof} (\textbf{Theorem \ref{th1}}) \\
Building upon the established lemmas, we have:
\begin{small}
\begin{align} 
&\| \theta^{t+1} - \theta^* \|^2 = \| \theta^t - K \eta_t \sum_{i \in S}  \omega_i d_i^{t} - \theta^* \|^2 \\
&=\| \theta^{t} - \theta^* -K \eta_t \sum_{i \in S}  \omega_i \nabla F_i(\theta_i^t)\\& - K \eta_t \sum_{i \in S}  \omega_i d_i^{t} + K \eta_t \sum_{i \in S}  \omega_i \nabla F_i(\theta_i^t) \|^2 \\
&= \| \theta^{t} - \theta^* -K \eta_t \sum_{i \in S}  \omega_i \nabla F_i(\theta_i^t)  \|^2\\& + K^2 {\eta_t}^2 \| \sum_{i \in S}  \omega_i (\nabla F_i(\theta_i^t)-d_i^{t}) \|^2\\
& + 2K \eta_t\langle \theta^t - \theta^* -K \eta_t \sum_{i \in S}  \omega_i \nabla F_i(\theta_i^t), \sum_{i \in S}  \omega_i (\nabla F_i(\theta_i^t) - d_i^t) \rangle 
\\
&= \| \theta^t - \theta^* \|^2 - 2K\eta_t \langle \theta^t - \theta^*, \sum_{i \in S}  \omega_i  \nabla F_i(\theta_i^t) \rangle \\
&+ 2K \eta_t\langle \theta^t - \theta^* -K \eta_t \sum_{i \in S}  \omega_i \nabla F_i(\theta_i^t), \sum_{i \in S}  \omega_i (\nabla F_i(\theta_i^t) - d_i^t) \rangle
\\&+  K^2\eta_t^2 \| \sum_{i \in S}  \omega_i  \nabla F_i(\theta_i^t) \|^2+K^2\eta_t^2 \| \sum_{i \in S}  \omega_i  (\nabla F_i(\theta_i^t)-d_i^{t}) \|^2. 
\end{align}
\end{small}

Now, after breaking down $\| \theta^{t+1} - \theta^* \|^2$ to several terms, we find the upper bounds for the following terms:
\begin{align}
    \mathcal{B}_{\text{grad}}&=-2K \eta_t \langle \theta^t - \theta^*, \sum_{i \in S}  \omega_i  \nabla F_i(\theta_i^t) \rangle,\\
    \mathcal{B}_{\text{noise}}&=2K \eta_t\langle \theta^t - \theta^* -K \eta_t \sum_{i \in S}  \omega_i \nabla F_i(\theta_i^t)\\&, \sum_{i \in S}  \omega_i (\nabla F_i(\theta_i^t) - d_i^{t}) \rangle,\\
    \mathcal{B}_{\text{step}}&= K^2\eta_t^2 \| \sum_{i \in S}  \omega_i  \nabla F_i(\theta_i^t) \|^2,\\
    \mathcal{B}_{\text{var}}&=K^2\eta_t^2 \| \sum_{i \in S}  \omega_i  (\nabla F_i(\theta_i^t)-d_i^{t}) \|^2. 
\end{align}
For $\mathcal{B}_{\text{grad}}$ we have,
\begin{align}
&\mathcal{B}_{\text{grad}}=-2K \eta_t \langle \theta^t - \theta^*, \sum_{i \in S}  \omega_i  \nabla F_i(\theta_i^t) \rangle \\&= -2K \eta_t \sum_{i \in S} \langle  \theta^t - \theta^*,   \omega_i  \nabla F_i(\theta_i^t) \rangle \\
&= -2K \eta_t \sum_{i \in S} \langle \theta^{t} - \theta_i^{t},   \omega_i  \nabla F_i(\theta_i^t) \rangle \\&- 2K \eta_t \sum_{i \in S} \langle \theta_i^{t} - \theta^{*},  \omega_i  \nabla F_i(\theta_i^t) \rangle  \\ & \leq
K \eta_t \sum_{i \in S} ( \frac{1}{\eta_t} ||\theta^t - \theta_i^t||^2+ \eta_t ||\omega_i  \nabla F_i(\theta_i^t)||^2 ) \label{eq45}\\&- 2K \eta_t \sum_{i \in S} \langle \theta_i^{t} - \theta^{*},   \omega_i  \nabla F_i(\theta_i^t) \rangle\\ 
\end{align}
\begin{align}
&\leq K  \sum_{i \in S}  ||\theta^t - \theta_i^t||^2 + K \eta_t^2 \sum_{i \in S} ||\omega_i \nabla F_i(\theta_i^t)||^2\\& - 2K \eta_t \sum_{i \in S} \langle \theta_i^{t} - \theta^{*},   \omega_i  \nabla F_i(\theta_i^t) \rangle \\ \label{49}&\leq 
16 G^2 {\eta_{t}}^2 K^3 m^2( \Sigma +1 ) \\&+ 2K \eta_t^2 L \sum_{i \in S} \omega_i(\nabla F_i(\theta_i^t)-F_i^*) \\&- 2K \eta_t \sum_{i \in S} \langle \theta_i^{t} - \theta^{*},   \omega_i  \nabla F_i(\theta_i^t) \rangle \\ &\leq
16 G^2 {\eta_{t}}^2 K^3 m^2 ( \Sigma +1 )\\&+ 2K \eta_t^2 L \sum_{i \in S} \omega_i(\nabla  F_i(\theta_i^t)-F_i^*) \\ &
-2K \eta_t \sum_{i \in S}  \omega_i\left[ F_i (\theta_i^t) - F_i(\theta^{*})+ \frac{\mu}{2} ||\theta_i^t - \theta^{*}||^2)) \right] \label{52}
\end{align}
where we use the AM-GM inequality in \eqref{eq45}.
Lemma \ref{lem1}, Lemma \ref{lem3} and the fact that $\omega_{i}<1$ is used in \eqref{49}.  \eqref{52} is due to $\mu$-convexity of $F_i$.

For $\mathcal{B}_{\text{noise}}$ in expectation, $\mathbb{E}[\mathcal{B}_{\text{noise}}] = 0$ due to the unbiased gradient. 

For $\mathcal{B}_{\text{step}}$, we determine the upper bound as follows using Lemma \ref{lem1}:
\begin{align}
\eta_t^2  \sum_{k \in S} \|\omega_k \nabla F_k(\theta_k^t)\|^2 &\leq \eta_t^2 \sum_{k \in S} \omega_k \|\nabla F_k(\theta_k^t)\|^2  
\\&\leq 2L\eta_t^2  \sum_{k \in S} \omega_k(F_k(\theta_k^t) - F^*_k).
\end{align}
Considering $\mathcal{B}_{\text{var}}$ in expectation, the upper bound is:
\begin{align}
\mathcal{B}_{\text{var}} &\leq K^2\eta_t^2  \sum_{i \in S} \|   (\nabla F_i(\theta_i^t)-d_i^{t}) \|^2 \\&= K^2\eta_t^2  \sum_{i \in S} \|    \frac{1}{K} \sum_{j=1}^K (g_i(y_{i,j}^t) -\nabla F_i(\theta_i^t))\|^2 \\&\leq 
\eta_t^2 m K^2  \sigma^2,\label{74}
\end{align}
where \eqref{74} results from Assumption \ref{assumpt:3}.

Integrating the established bounds for $\mathcal{B}_{\text{grad}}, \mathcal{B}_{\text{noise}}, \mathcal{B}_{\text{step}},$ and $\mathcal{B}_{\text{var}}$, it follows that the expected value is constrained as:
\begin{align}
  & \mathbb{E}[\| \theta^{t+1} - \theta^* \|^2 ]\leq \mathbb{E}[ \| \theta^{t} - \theta^* \|^2] +\eta_t^2 m K^2  \sigma^2\\& + 2L\eta_t^2 (K+1) \mathbb{E} [\sum_{k \in S} \omega_k(F_k(\theta_k^t) - F^*_k)] \\ &+ 16 G^2 {\eta_{t}}^2 K^3 m^2 ( \Sigma +1 ) \\&-2K \eta_t \mathbb{E}[\sum_{i \in S}  \omega_i\left[ F_i (\theta_i^t) - F_i(\theta^{*})+ \frac{\mu}{2} ||\theta_i^t - \theta^{*}||^2)) \right] ] \\ \label{79}&\leq  (1-K \mu \eta_t) \mathbb{E}[\| \theta^{t} - \theta^* \|^2] + 16 G^2 {\eta_{t}}^2 K^3 m^2( \Sigma +1 )\\&+\eta_t^2 m K^2  \sigma^2
  +2L\eta_t^2 (K+1) \mathbb{E} [\sum_{k \in S} \omega_k(F_k(\theta_k^t) - F^*_k)]\\&
  -2K \eta_t \mathbb{E} [\sum_{i \in S}  \omega_i ( F_i (\theta_i^t) - F_i(\theta^{*}))  ],
\end{align}
where \eqref{79} is due to $\mathbb{E}[\| \theta^{t} - \theta^* \|^2]\leq (\sum_{i \in S}\omega_i)(\sum_{i \in S}  \omega_i  ||\theta_i^t - \theta^{*}||^2)=\sum_{i \in S}  \omega_i  ||\theta_i^t - \theta^{*}||^2$.

We define $\mathcal{B}_{\text{gap}}
$ as follows and perform some algebraic calculations:
\begin{align}
  &\mathcal{B}_{\text{gap}}= 2L\eta_t^2 (K+1) \mathbb{E} [\sum_{k \in S} \omega_k(F_k(\theta_k^t) - F^*_k)] \\&-2K \eta_t \mathbb{E}[\sum_{i \in S}  \omega_i ( F_i (\theta_i^t) - F_i(\theta^{*}))  ]\\&= \mathbb{E} [2L\eta_t^2 (K+1)\sum_{k \in S} \omega_k(F_k(\theta_k^t) - F^*_k) \\&-2K \eta_t \sum_{i \in S}  \omega_i ( F_i (\theta_i^t) - F_i(\theta^{*}))]\\&=\mathbb{E}[2L\eta_t^2 (K+1)\sum_{k \in S} \omega_k F_k(\theta_k^t) -2K \eta_t \sum_{i \in S}  \omega_i  F_i (\theta_i^t)\\& -2L\eta_t^2 (K+1)\sum_{k \in S} \omega_k F^*_k+2K \eta_t \sum_{i \in S}  \omega_i  F_i(\theta^{*})]\\&=
 \mathbb{E}[ 2K\eta_t (\frac{L\eta_t(K+1)}{K} -1) \sum_{k \in S} \omega_k(F_k(\theta_k^t)-F_k^*)] \\&+
 2K \eta_t \mathbb{E}[\sum_{k \in S} \omega_k(F_k(\theta^*)-F_k^*)]
\end{align}
Let us define $\mathcal{B}_{\text{drift}}
$ as:
\begin{align}
    \mathcal{B}_{\text{drift}}&= 2K\eta_t (\frac{L\eta_t(K+1)}{K} -1) \sum_{k \in S} \omega_k(F_k(\theta_k^t)-F_k^*),\\
     a_t&=  2K\eta_t (1-\frac{L\eta_t(K+1)}{K} )
\end{align}

Assuming $\eta_t< \frac{1}{2L(K+1)}$, we have:
\begin{align}
    &\mathcal{B}_{\text{drift}}= -a_t \sum_{k \in S} \omega_k(F_k(\theta_k)-F_k(\theta^t)+F_k(\theta^t)-F_k^*)\\&=-a_t \sum_{k \in S} \omega_k(F_k(\theta_k)-F_k(\theta^t))\\&-a_t \sum_{k \in S} \omega_k(F_k(\theta^t)-F_k^*)\\&\leq
    -a_t \sum_{k \in S}[ \langle \omega_k\nabla  F_k(\theta^t), \theta_k^t -\theta^t \rangle + \frac{\mu}{2} ||  \theta_k -\theta^t ||^2]\\&-a_t \sum_{k \in S} \omega_k(F_k(\theta^t)-F_k^*) \label{76}\\&\leq
    a_t \sum_{k \in S}[\eta_t L \omega_k(F_k(\theta^t)-F_k^*)\\&+(\frac{1}{2\eta_t}-\frac{\mu}{2})||  \theta_k -\theta^t ||^2] \label{78}\\&-a_t \sum_{k \in S} \omega_k(F_k(\theta^t)-F_k^*)\\&=-a_t(1-\eta_tL)\sum_{k \in S} \omega_k(F_k(\theta^t)-F_k^*)\\&+a_t(\frac{1}{2\eta_t}-\frac{\mu}{2})\sum_{k \in S}||  \theta_k -\theta^t ||^2,
\end{align}
where \eqref{76} is due to $\mu$-convexity of $F_k$ and \eqref{78} is based on Lemma \ref{lem1}. Using the bound identified for $\mathcal{B}_{\text{gap}}
$, we can now determine the bound for $\mathcal{B}_{\text{drift}}$ as:
\begin{align}
   & \mathcal{B}_{\text{drift}} \leq -a_t(1-\eta_tL)\sum_{k \in S} \omega_k(F_k(\theta^t)-F_k^*)\\&+a_t(\frac{1}{2\eta_t}-\frac{\mu}{2})||  \theta_k -\theta^t ||^2\\&+2K \eta_t\mathbb{E}[\sum_{k \in S} \omega_k(F_k(\theta^*)-F_k^*)] \\ &\leq \label{84}
   16 G^2 {\eta_{t}}^2 K^2 m^2 ( \Sigma +1 )\\&-a_t(1-\eta_tL)\sum_{k \in S} \omega_k(F_k(\theta^t)-F_k^*)\\&+2K \eta_t\mathbb{E}[\sum_{k \in S} \omega_k(F_k(\theta^*)-F_k^*)]\\&=16 G^2 {\eta_{t}}^2 K^2 m^2( \Sigma +1 )\\& -a_t(1-\eta_tL)\mathbb{E}[\rho(h; \theta^t)( F(\theta^t) - \sum_{k=1}^{n} q_k F_k^*)]\\& + 2K \eta_t\mathbb{E}[\rho(h; \theta^*)( F^* - \sum_{k=1}^{n} q_k F_k^*)] \\ & \leq \label{89} 16 G^2 {\eta_{t}}^2 K^2m^2( \Sigma +1 )+ 2K\eta_t  \overline{\rho} \Gamma \\&-a_t(1-\eta_tL) \underline{\rho} \mathbb{E}[ F(\theta^t) - \sum_{k=1}^{n} q_k F_k^*],
\end{align}
where \eqref{84} is due to Lemma \ref{lem3} and the fact that $a_t(\frac{1}{2\eta_t}-\frac{\mu}{2})\leq 1$. \eqref{89} is due to the definitions of $\rho(h)$, $\overline{\rho}$, and $\underline{\rho}$ in \ref{def1}.
Defining $\mathcal{B}_{\text{global}}
$ as follows, we have:

\begin{align}
   & \mathcal{B}_{\text{global}}=-a_t(1-\eta_tL) \underline{\rho} \mathbb{E}[ F(\theta^t) - \sum_{k=1}^{n} p_k F_k^*]\\&=
    -a_t(1-\eta_tL) \underline{\rho} \sum_{k=1}^{n} p_k (\mathbb{E}[F_k(\theta^{t})]-F^*+F^* -F_k^*)\\&=
   -a_t(1-\eta_tL) \underline{\rho} \sum_{k=1}^{n} p_k (\mathbb{E}[F_k(\theta^{t})]-F^*)\\&-a_t(1-\eta_tL) \underline{\rho} \Gamma\\&=-a_t(1-\eta_tL) \underline{\rho}  (\mathbb{E}[F(\theta^{t})]-F^*)-a_t(1-\eta_tL) \underline{\rho} \Gamma\\&\leq \label{94}
   \frac{-a_t(1-\eta_tL) \underline{\rho} \mu}{2} \mathbb{E}[||\theta^t-\theta^*||^2]-a_t(1-\eta_tL) \underline{\rho} \Gamma \\ &\leq -(K-1) \eta_t \underline{\rho} \mu \mathbb{E}[||\theta^t-\theta^*||^2]-2K\eta_t\underline{\rho} \Gamma \\&+6KL\eta_t^2\underline{\rho} \Gamma,
\end{align}

where \eqref{94} is due to $\mu$-convexity and the fact that $\frac{-a_t(1-\eta_tL) }{2} \leq -(K-1) \eta_t $.
Therefore, now we can bound $\mathcal{B}_{\text{gap}}
$ as:
\begin{align}
   \mathcal{B}_{\text{gap}} &\leq-(K-1) \eta_t \underline{\rho} \mu \mathbb{E}[||\theta^t-\theta^*||^2]\\&+ 2K\eta_t(\overline{\rho}-\underline{\rho})\Gamma\\&+\eta_t^2(6KL\underline{\rho} \Gamma+16 G^2  K^2m^2( \Sigma +1 ))
\end{align}
Thus $\mathbb{E}[||\theta^{t+1}-\theta^*||^2]$ can be bounded as:

\begin{align}
  &  (1-\eta_t K\mu(1+\frac{K-1}{K}\underline{\rho}))\mathbb{E}[||\theta^t-\theta^*||^2]\\&+2K\eta_t(\overline{\rho}_-\underline{\rho})\Gamma \\&+\eta_t^2(6KL\underline{\rho} \Gamma+16 G^2  K^2m^2(K +1)( \Sigma +1 )+ m K^2  \sigma^2)
\end{align}
Let us define $\Delta_{t+1}=\mathbb{E}[||\theta^{t+1}-\theta^*||^2]$, then:
\begin{align}
    \Delta_{t+1} \leq (1-\eta_t \mu K B) \Delta_t +\eta_t^2 C+ \eta_t D,
\end{align}
where $B=1+\frac{K-1}{K}\underline{\rho}$, $C=6KL\underline{\rho} \Gamma+16 G^2  K^2m^2(K+1)( \Sigma +1 )+ m K^2  \sigma^2$, and $D=2K(\overline{\rho}-\underline{\rho})\Gamma$.

 By setting $\Delta_t \leq \frac{\psi}{t+\gamma}, \eta_t = \frac{\beta}{t+\gamma }$, $\beta > \frac{1}{ \mu K B},$ assuming $(\beta = \frac{1}{ \mu K})$ and $\gamma>0$ by induction we have that:
 \begin{small}
\begin{equation}
\psi = \max \left\{ \gamma \|\theta^0 - \theta^*\|^2, \frac{1}{\beta \mu K B-1 }   (\beta^2 C + D\beta(t + \gamma)) \right\}
\end{equation}
\end{small}
Then by the L-smoothness of \(F(\theta)\), we have that:
\begin{equation}
\mathbb{E}[F(\theta(t))] - F^* \leq \frac{L}{2} \Delta_t \leq \frac{L}{2}  \frac{\psi}{t+\gamma}
\end{equation}\end{proof}
\begin{proof} (\textbf{Theorem \ref{th3}}) \\
Since exactly $m$ clients are selected in each round, each subset $S$ of size $m$ occurs with probability $\tfrac{1}{\binom{n}{m}}$. By definition, for $i \in S$ we have $\omega_i = \tfrac{d_i}{\sum_{j \in S} d_j}$ and for $i \notin S$, $\omega_i = 0$. Observe that
\[
\sum_{i=1}^n \omega_i^2 \;=\;
\sum_{i \in S}
\biggl(\frac{d_i}{\sum_{j \in S} d_j}\biggr)^2,
\]
Therefore,
\[
\mathbb{E}\Bigl[\sum_{i=1}^n \omega_i^2\Bigr]
\;=\;
\frac{1}{\binom{n}{m}} \sum_{\substack{S\subseteq\{1,\dots,n\}\\|S|=m}}
\;\sum_{i \in S}
\frac{d_i^2}{\Bigl(\sum_{j\in S}d_j\Bigr)^2}.
\]
Next, for each fixed client $i$, its mean weight is
\begin{small}
\[
\mathbb{E}[\omega_i]
\;=\;
\sum_{\substack{S:\, i \in S}} 
\frac{d_i}{\sum_{j \in S} d_j}
\;\times\;\frac{1}{\binom{n}{m}}
\;=\;
\frac{1}{\binom{n}{m}}
\sum_{\substack{S: \, i \in S}} 
\frac{d_i}{\sum_{j \in S} d_j}.
\]
\end{small}
Hence
\[
\sum_{i=1}^n \bigl(\mathbb{E}[\omega_i]\bigr)^2
\;=\;
\sum_{i=1}^n
\biggl(
  \frac{1}{\binom{n}{m}}
  \sum_{\substack{S:\, i\in S}}
  \frac{d_i}{\sum_{j \in S} d_j}
\biggr)^2.
\]
And we have
\[
\Sigma= \sum_{i=1}^n \mathrm{Var}[\omega_i]
\;=\;
\sum_{i=1}^n \mathbb{E}[\omega_i^2]
\;-\;
\sum_{i=1}^n \bigl(\mathbb{E}[\omega_i]\bigr)^2.
\]
Combining these pieces yields the claimed formula.
\end{proof}

\begin{proof} (\textbf{Theorem \ref{th4}})

We begin by defining the random variable \(X_i\), which represents the number of times client \(i\) is selected out of the \(k\) draws. Since the selection is performed with replacement, \(X_i\) follows a binomial distribution with parameters \(m\) and \(q_i\). Specifically, \(X_i\) is given by:
$$
\mathbb{P}(X_i = r) = \binom{m}{r} q_i^r (1 - q_i)^{m-r},
$$
where \(r \in \{0, 1, \dots, m\}\).

The weight \(\omega_i\) for client \(i\) is defined as \(\omega_i = \frac{r}{m}\) if client \(i\) is selected \(r\) times, and \(\omega_i = 0\) if client \(i\) is not selected. The expected weight \(\mathbb{E}[\omega_i]\) is given by:
$$
\mathbb{E}[\omega_i] = \sum_{r=0}^m \mathbb{P}(X_i = r) \cdot \frac{r}{m}.
$$
Substituting the probability mass function of \(X_i\), we have:
$$
\mathbb{E}[\omega_i] = \frac{1}{m} \sum_{r=0}^m r \cdot \binom{m}{r} q_i^r (1 - q_i)^{m-r}.
$$
This expression can be simplified by recognizing that the sum represents the expected value of a binomial distribution:
$$
\mathbb{E}[X_i] = \sum_{r=0}^m r \cdot \mathbb{P}(X_i = r) = mq_i.
$$
Thus, the expected weight for client \(i\) simplifies to:
$$
\mathbb{E}[\omega_i] = \frac{mq_i}{m} = q_i.
$$

Next, we calculate the expected square of the weight \(\mathbb{E}[\omega_i^2]\):
$$
\mathbb{E}[\omega_i^2] = \sum_{r=0}^m \mathbb{P}(X_i = r) \cdot \left(\frac{r}{m}\right)^2.
$$
Substituting the probability mass function of \(X_i\), we obtain:
$$
\mathbb{E}[\omega_i^2] = \frac{1}{m^2} \sum_{r=0}^m r^2 \cdot \binom{m}{r} q_i^r (1 - q_i)^{m-r}.
$$
This sum corresponds to the second moment of the binomial distribution. The second moment \(\mathbb{E}[X_i^2]\) is given by:
$$
\mathbb{E}[X_i^2] = \text{Var}(X_i) + \left(\mathbb{E}[X_i]\right)^2,
$$
where the variance of \(X_i\) is \(\text{Var}(X_i) = mq_i(1 - q_i)\). Therefore:
$$
\mathbb{E}[X_i^2] = mq_i(1 - q_i) + (mq_i)^2.
$$
Substituting this into the expression for \(\mathbb{E}[\omega_i^2]\), we have:
$$
\mathbb{E}[\omega_i^2] = \frac{q_i(1 - q_i)}{m} + q_i^2.
$$

The variance of the weight \(\omega_i\) is then given by:
\begin{align}
\text{Var}[\omega_i] &= \mathbb{E}[\omega_i^2] - \left(\mathbb{E}[\omega_i]\right)^2 \\ &= \left(\frac{q_i(1 - q_i)}{m} + q_i^2\right) - q_i^2 = \frac{q_i(1 - q_i)}{m}.
\end{align}

Therefore, we can calculate $\Sigma$ as:
$$
\Sigma=\sum_{i=1}^{n} \text{Var}[\omega_i]= \sum_{i=1}^n \frac{q_i(1 - q_i)}{m}.
$$
\end{proof}
\begin{proof} (\textbf{Theorem \ref{th5}})
Assume each client's age follows a Markov chain with a unique stationary distribution
$\bigl(\pi_0,\pi_1,\dots,\pi_{m'}\bigr)$, so that the \emph{marginal} probability that a client is selected in steady state is
\[
 p_{\mathrm{avg}} \;=\; \sum_{a=0}^{m'} \pi_a\,p_a.
\]
We only consider these steady-state rounds. When different clients' selections are independent, the random variable $\lvert S\rvert$ follows a $\mathrm{Binomial}$ $\!\bigl(n,p_{\mathrm{avg}}\bigr)$ distribution, and assuming also $|S|>0$ using forced selection of a random client in the event no client is selected.

By definition,
\[
  \omega_i \;=\;
  \begin{cases}
     \dfrac{1}{|S|}, & \text{if } i\in S,\\
     0,              & \text{if } i\notin S.
  \end{cases}
\]
Thus, whenever $|S|\ge1$, we have $\sum_{i=1}^n \omega_i \;=\; 1.$
Hence,
\[
  \sum_{i=1}^n \mathbb{E}[\omega_i]
  \;=\;
  \mathbb{E}\Bigl[\sum_{i=1}^n \omega_i\Bigr]
  \;=\; 1,
\]
which implies
\[
  \mathbb{E}[\omega_i] \;=\; \frac{1}{n}\quad\text{for each } i.
\]

\medskip

Next, we compute the second moments.  Note that if $|S|\ge1$,
\[
  \sum_{i=1}^n \omega_i^2
  \;=\;
  \sum_{i \in S} \frac{1}{|S|^2}
  \;=\;
  \frac{|S|}{|S|^2}
  \;=\;
  \frac{1}{|S|}.
\]
Therefore,
\[
  \sum_{i=1}^n \mathbb{E}[\omega_i^2]
  \;=\;
  \mathbb{E}\!\Bigl[\tfrac{1}{|S|}\Bigr].
\]
By symmetry, $\mathbb{E}[\omega_i^2]$ is the same for all $i$, so
\[
  n\,\mathbb{E}[\omega_i^2]
  \;=\;
  \mathbb{E}\!\Bigl[\tfrac{1}{|S|}\Bigr].
\]
Then the variance of each $\omega_i$ is
\[
  \mathrm{Var}[\omega_i]
  \;=\;
  \mathbb{E}[\omega_i^2] \;-\; (\mathbb{E}[\omega_i])^2
  \;=\;
  \frac{1}{n}\,\mathbb{E}\!\Bigl[\tfrac{1}{|S|}\Bigr]
  \;-\;
  \left(\frac{1}{n}\right)^2.
\]
Summing over $i=1,\dots,n$ yields
\[
  \sum_{i=1}^n \mathrm{Var}[\omega_i]
  \;=\;
  \sum_{i=1}^n
    \Bigl(\mathbb{E}[\omega_i^2] \;-\; (\mathbb{E}[\omega_i])^2\Bigr)
  \;=\;
  \mathbb{E}\!\Bigl[\tfrac{1}{|S|}\Bigr]
  \;-\;
  \frac{1}{n}.
\]
Finally, since $|S| \sim \mathrm{Binomial}\bigl(n,p_{\mathrm{avg}}\bigr)$ under the Markov-by-age policy and conditional independence, we have
\[
  \mathbb{E}\!\Bigl[\tfrac{1}{|S|}\Bigr]
  \;=\;
  \sum_{s=1}^{n}
    \frac{1}{s}
    \,\binom{n}{s}
    \,(p_{\mathrm{avg}})^s
    \,\bigl(1 - p_{\mathrm{avg}}\bigr)^{n-s}.
\]
Thus,
\[
  \Sigma
  \;=\;
  \mathbb{E}\!\Bigl[\tfrac{1}{|S|}\Bigr]
  \;-\;\frac{1}{n}
  \;=\;
  \sum_{s=1}^n
    \frac{1}{s}
    \,\binom{n}{s}
    \,(p_{\mathrm{avg}})^s
    \bigl(1 - p_{\mathrm{avg}}\bigr)^{n-s}
  \;-\;
  \frac{1}{n}.
\]
This completes the proof.
\end{proof}

\printbibliography
\end{document}